\documentclass{article}
\usepackage[margin=1in]{geometry}

\usepackage[usenames,dvipsnames]{xcolor}
\usepackage{float}						
\usepackage{caption}
\usepackage{subcaption}					

\usepackage{algorithm}
\usepackage[noend]{algorithmic}

\usepackage{amsmath,amsbsy,amsfonts,amssymb,amsthm,physics}
\usepackage{mathtools}
\usepackage{graphicx}
\usepackage{natbib}
\usepackage{algorithm, algorithmic}
\usepackage{color}
\usepackage{float}						
\usepackage{caption}
\usepackage{subcaption}
\usepackage{enumitem}
\usepackage{thmtools}
\usepackage{thm-restate}
\usepackage{cleveref}

\newcommand{\algorithmicbreak}{\textbf{break}}
\newcommand{\BREAK}{\STATE \algorithmicbreak}

\theoremstyle{plain}
\newtheorem{theorem}{Theorem}
\newtheorem{claim}{Claim}

\newtheorem{assumption}{Assumption}
\newtheorem{lemma}[theorem]{Lemma}

\theoremstyle{definition}
\newtheorem{condition}{Condition}
\newtheorem{definition}{Definition}

\newtheorem{remark}{Remark}
\allowdisplaybreaks

\newcommand{\x}{\mathbf{x}}
\renewcommand{\O}{\mathcal{O}}
\newcommand{\bDelta}{\mathbf{\Delta}}
\DeclareMathOperator*{\argmin}{argmin}

\title{Stochastic Cubic Regularization for Fast Nonconvex Optimization}
\author{
Nilesh Tripuraneni\footnote{Equal contribution.}\phantom{$^\ast$} \quad
Mitchell Stern$^\ast$ \quad
Chi Jin \quad
Jeffrey Regier \quad
Michael I. Jordan
\\
\texttt{\{nilesh\_tripuraneni,mitchell,chijin,regier\}@berkeley.edu} \\
\texttt{jordan@cs.berkeley.edu}
\\ \\
University of California, Berkeley
}

\begin{document}

\maketitle


\begin{abstract}
This paper proposes a stochastic variant of a classic algorithm---the cubic-regularized Newton method \citep{nesterov2006cubic}. The proposed algorithm efficiently escapes saddle points and finds approximate local minima for general smooth, nonconvex functions in only $\mathcal{\tilde{O}}(\epsilon^{-3.5})$ stochastic gradient and stochastic Hessian-vector product evaluations. The latter can be computed as efficiently as stochastic gradients. This improves upon the $\mathcal{\tilde{O}}(\epsilon^{-4})$ rate of stochastic gradient descent. Our rate matches the best-known result for finding local minima without requiring any delicate acceleration or variance-reduction techniques. 

  
\end{abstract}


\section{Introduction}\label{sec:intro}
We consider the problem of nonconvex optimization in the stochastic approximation framework \citep{robbins1951stochastic}:
\begin{align}
  \min_{\mathbf{x} \in \mathbb{R}^d} f(\mathbf{x}) \coloneqq \mathbf{E}_{\xi \sim \mathcal{D}} [f(\mathbf{x}; \xi)]. \label{eq:SA}
\end{align}
In this setting, we only have access to the stochastic function $f(\mathbf{x}; \xi)$, where the random variable $\xi$ is sampled from an underlying distribution $\mathcal{D}$. The task is to optimize the expected function $f(\mathbf{x})$, which in general may be nonconvex. This framework covers a wide range of problems, including the offline setting where we minimize the empirical loss over a fixed amount of data, and the online setting where data arrives sequentially. One of the most prominent applications of stochastic optimization has been in large-scale statistics and machine learning problems, such as the optimization of deep neural networks.

Classical analysis in nonconvex optimization only guarantees convergence to a first-order stationary point (i.e., a point $\mathbf{x}$ satisfying $\norm{\nabla f(\mathbf{x})}=0$),
which can be a local minimum, a local maximum, or a saddle point.
This paper goes further, proposing an algorithm that escapes saddle points and converges to a local minimum.
A local minimum is defined as a point $\mathbf{x}$ satisfying
$\norm{\nabla f(\mathbf{x})} = 0$ and $\nabla^2 f(\mathbf{x}) \succeq 0$.
Finding such a point is of special interest for a large class of statistical learning problems where local minima are global or near-global solutions (e.g. \citet{choromanska2014loss, sun2016complete, sun2016geometric, ge2017no}).

Among first-order stochastic optimization algorithms, stochastic gradient descent (SGD) is perhaps the simplest and most versatile.
While SGD is computationally inexpensive, the best current guarantee for finding an $\epsilon$-approximate local minimum (see Definition~\ref{def: eps-2}) requires $\mathcal{O}(\epsilon^{-4} \text{poly}(d))$ iterations \citep{ge2015escaping}, which is inefficient in the high-dimensional regime.

In contrast, second-order methods which have access to the Hessian of $f$ can exploit curvature to more effectively escape saddles and arrive at local minima. Empirically, second-order methods are known to perform better than first-order methods on a variety of nonconvex problems \citep{rattray1998natural, martens2010deep, regier2017fast}. However, many classical second-order algorithms need to construct the full Hessian, which can be prohibitively expensive when working with large models. Recent work has therefore explored the use of Hessian-vector products $\nabla^2 f(\mathbf{x}) \cdot \mathbf{v}$, which can be computed as efficiently as gradients in many cases including neural networks \citep{pearlmutter1994fast}.

Several algorithms incorporating Hessian-vector products \citep{carmon2016accelerated, agarwal2016finding} have been shown to achieve faster convergence rates than gradient descent in the non-stochastic setting. However, in the stochastic setting where we only have access to stochastic Hessian-vector products, significantly less progress has been made.
This leads us to ask the central question of this paper:
\textbf{Can stochastic Hessian-vector products help to speed up nonconvex optimization?}

In this work, we answer in the affirmative, proposing a stochastic optimization method that utilizes stochastic gradients and Hessian-vector products to find an $\epsilon$-second-order stationary point using only $\mathcal{\tilde{O}}(\epsilon^{-3.5})$ oracle evaluations, where $\tilde{\O}(\cdot)$ hides poly-logarithmic factors. Our rate improves upon the $\mathcal{\tilde{O}}(\epsilon^{-4})$ rate of stochastic gradient descent, and matches the best-known result for finding local minima
without the need for
any delicate acceleration or variance reduction techniques (see Section \ref{sec:related} for details).

The proposed algorithm in this paper is based on a classic algorithm in the non-stochastic setting---the cubic regularized Newton method (or cubic regularization)
\citep{nesterov2006cubic}. This algorithm is a natural extension of gradient descent that incorporates Hessian information.
Whereas gradient descent finds the minimizer of a local second-order Taylor expansion,
\begin{equation*}
\mathbf{x}_{t+1}^{\text{GD}} = \argmin_{\mathbf{x}} \left[ f(\mathbf{x}_t) + \nabla f(\mathbf{x}_t)^\top (\mathbf{x} - \mathbf{x}_t) + \frac{\ell}{2}\norm{\mathbf{x} - \mathbf{x}_t}^2 \right] ,
\end{equation*}
the cubic regularized Newton method finds the minimizer of a local third-order Taylor expansion,
\begin{align*}
\mathbf{x}_{t+1}^{\text{Cubic}} =& \argmin_{\mathbf{x}} \left[ f(\mathbf{x}_t) + \nabla f(\mathbf{x}_t)^\top (\mathbf{x} - \mathbf{x}_t) + \frac{1}{2} (\mathbf{x} - \mathbf{x}_t)^\top \nabla^2 f(\mathbf{x}_t) (\mathbf{x} - \mathbf{x}_t) + \frac{\rho}{6} \norm{\mathbf{\mathbf{x} - \mathbf{x}_t}}^3 \right] .
\end{align*}
We provide a stochastic variant of this classic algorithm, bridging the gap between its use in the non-stochastic and stochastic settings.

In contrast to prior work, we present a fully stochastic cubic-regularized Newton method: both gradients and Hessian-vector products are observed with noise.
Additionally, we provide a non-asymptotic analysis of its complexity.

\subsection{Related Work}\label{sec:related}
There has been a recent surge of interest in optimization methods that can escape saddle points and find $\epsilon$-approximate local minima (see Definition \ref{def: eps-2}) in various settings. We provide a brief summary of these results. All iteration complexities in this section are stated in terms of finding approximate local minima, and only highlight the dependency on $\epsilon$ and $d$.

\subsubsection{General Case}
This line of work optimizes over a general function $f$ without any special structural assumptions. In this setting, the optimization algorithm has direct access to the gradient or Hessian oracles at each iteration. The work of \citet{nesterov2006cubic} first proposed the cubic-regularized Newton method, which requires $\mathcal{O}(\epsilon^{-1.5})$ gradient and Hessian oracle calls to find an $\epsilon$-second-order stationary point. Later, the ARC algorithm \citep{cartis2011adaptiveb} and trust-region methods \citep{curtis2014trust} were also shown to achieve the same guarantee with similar Hessian oracle access. However, these algorithms rely on having access to the full Hessian at each iteration, which is prohibitive in high dimensions.

Recently, instead of using the full Hessian, \citet{carmon2016gradient} showed that using a gradient descent solver for the cubic regularization subproblem allows their algorithm to find $\epsilon$-second-order stationary points in $\mathcal{\tilde{O}}(\epsilon^{-2})$ Hessian-vector product evaluations. With acceleration techniques, the number of Hessian-vector products can be reduced to $\mathcal{\tilde{O}}(\epsilon^{-1.75})$  \citep{carmon2016accelerated, agarwal2016finding, royer2017complexity}.

Meanwhile, in the realm of entirely Hessian-free results, \citet{jin2017escape} showed that a simple variant of gradient descent can find $\epsilon$-second stationary points in $\mathcal{\tilde{O}}(\epsilon^{-2})$ gradient evaluations.
\subsubsection{Finite-Sum Setting}
In the finite-sum setting (also known as the offline setting) where $f(\mathbf{x}) \coloneqq \frac{1}{n} \sum_{i=1}^{n} f_{i}(\mathbf{x})$, one assumes that algorithms have access to the individual functions $f_i$. In this setting, variance reduction techniques can be exploited \citep{johnson2013accelerating}.
\citet{agarwal2016finding} give an algorithm requiring $\tilde{\O}(\frac{nd}{\epsilon^{3/2}} + \frac{n^{3/4} d}{\epsilon^{7/4}})$ Hessian-vector oracle calls to find an $\epsilon$-approximate local minimum. A similar result is achieved by the algorithm proposed by \citet{reddi2017generic}.

\subsubsection{Stochastic Approximation}
The framework of stochastic approximation where $f(\mathbf{x}) \coloneqq \mathbf{E}_{\xi \sim \mathcal{D}} [f(\mathbf{x}; \xi)]$ only assumes access to a stochastic gradient and Hessian via $f(\x; \xi)$.
In this setting, \citet{ge2015escaping} showed that the total gradient iteration
complexity for SGD to find an $\epsilon$-second-order stationary point was $\mathcal{O}(\epsilon^{-4} \text{poly}(d))$.
More recently, \citet{kohler2017sub} consider a subsampled version of the cubic regularization algorithm, but do not provide a non-asymptotic analysis for their algorithm to find an approximate local minimum; they also assume access to exact (expected) function values at each iteration which are not available in the fully stochastic setting. \citet{xu2017newton} consider the case of stochastic Hessians,
but also require access to exact gradients and function values at each iteration.
Very recently, \citet{allen2017natasha} proposed an algorithm with a mechanism exploiting variance reduction that finds a second-order stationary point with $\tilde{\O}(\epsilon^{-3.5})$%
\footnote{\label{footnote:allen-zhu}The original paper reports a rate of $\mathcal{\tilde{O}}(\epsilon^{-3.25})$ due to a different definition of $\epsilon$-second-order stationary point, $\lambda_{\min}(\nabla^2 f(x)) \ge -\O(\epsilon^{1/4})$ , which is weaker than the standard definition as in Definition~\ref{def: eps-2}.}
Hessian-vector product evaluations. We note that our result matches this best result so far using a simpler approach without any delicate acceleration or variance-reduction techniques.
See Table~\ref{tab:comparison} for a brief comparison.

\renewcommand{\arraystretch}{1.5}
\begin{table*}[t]
	\centering
	\begin{tabular}{|c|c|c|c|}
		\hline
		\textbf{Method} & \textbf{Runtime} & \textbf{Variance Reduction} \\
		\hline
    {Stochastic Gradient Descent \citep{ge2015escaping}} & $\mathcal{O}(\epsilon^{-4} \text{poly}(d))$ & not needed \\
    \hline
    Natasha 2 \citep{allen2017natasha} & $\mathcal{\tilde{O}}(\epsilon^{-3.5})^{1}$ & needed \\
		\hline
    \textbf{Stochastic Cubic Regularization (this paper)} & $ \mathcal{\tilde{O}}(\epsilon^{-3.5})$ & not needed \\
    \hline
	\end{tabular}
	\caption{Comparison of our results to existing results for stochastic, nonconvex optimization with provable convergence to approximate local minima.}
	\label{tab:comparison}
\end{table*}


\section{Preliminaries}\label{sec:prelim}
We are interested in stochastic optimization problems of the form
\begin{align}
  \min_{\mathbf{x} \in \mathbb{R}^d} f(\mathbf{x}) \coloneqq \mathbb{E}_{\xi \sim \mathcal{D}} [f(\mathbf{x}; \xi)],
\end{align}
where $\xi$ is a random variable with distribution $\mathcal{D}$. In general, the function $f(\mathbf{x})$ may be nonconvex. This formulation covers both the standard offline setting where the objective function can be expressed as a finite sum of $n$ individual functions $f(\mathbf{x}, \xi_i)$, as well as the online setting where data arrives sequentially.

Our goal is to minimize the function $f(\mathbf{x})$ using only stochastic gradients $\nabla f(\mathbf{x}; \xi)$ and stochastic Hessian-vector products $\nabla^2 f(\mathbf{x}; \xi) \cdot \mathbf{v}$, where $\mathbf{v}$ is a vector of our choosing. Although it is expensive and often intractable in practice to form the entire Hessian, computing a Hessian-vector product is as cheap as computing a gradient when our function is represented as an arithmetic circuit~\citep{pearlmutter1994fast}, as is the case for neural networks.

\textbf{Notation}:
We use bold uppercase letters $\mathbf{A}, \mathbf{B}$ to denote matrices and bold lowercase letters $\mathbf{x}, \mathbf{y}$ to denote vectors.
For vectors we use $\norm{\cdot}$ to denote the $\ell_2$-norm, and for matrices we use $\norm{\cdot}$ to denote the spectral norm and $\lambda_{\min}(\cdot)$ to denote the minimum eigenvalue. Unless otherwise specified, we use the notation $\O(\cdot)$ to hide only absolute constants which do not depend on any problem parameter, and the notation $\tilde{\O}(\cdot)$ to hide only absolute constants and logarithmic factors.

\subsection{Assumptions}
Throughout the paper, we assume that the function $f(\mathbf{x})$ is bounded below by some optimal value $f^*$.
We also make following assumptions about function smoothness:
\begin{samepage}
    \begin{assumption}
      The function $f(\mathbf{x})$ has
      \begin{itemize} [itemsep=.5em, nosep]
      \item
      $\ell$-Lipschitz gradients: for all $\mathbf{x}_1$ and $\mathbf{x}_2$,
      \begin{align}
      \norm{\nabla f(\mathbf{x}_1) - \nabla f(\mathbf{x}_2)]} \leq \ell \norm{\mathbf{x}_1 - \mathbf{x}_2}; \notag
      \end{align}
      \item
      $\rho$-Lipschitz Hessians: for all $\mathbf{x}_1$ and $\mathbf{x}_2$,
      \begin{align}
      \norm{\nabla^2 f(\mathbf{x}_1) - \nabla^2 f(\mathbf{x}_2)} \leq \rho \norm{\mathbf{x}_1 - \mathbf{x}_2}. \notag
      \end{align}
      \end{itemize}
      \label{def: A1}
    \end{assumption}
\end{samepage}
The above assumptions state that the gradient and the Hessian cannot change dramatically in a small local area, and are standard in prior work on escaping saddle points and finding local minima.

Next, we make the following variance assumptions about stochastic gradients and stochastic Hessians:
    \begin{assumption} The function $f(\mathbf{x}, \xi)$ has
      \begin{itemize} [itemsep=.5em, nosep]
      \item for all $\mathbf{x}$,
      $\mathbb{E} \left[\norm{\nabla f(\mathbf{x}, \xi) - \nabla f(\mathbf{x})}^2 \right] \leq \sigma_1^2$ and $\norm{\nabla f(\mathbf{x}, \xi) - \nabla f(\mathbf{x})} \leq M_1$ a.s.;
      \item for all $\mathbf{x}$,
      $\mathbb{E} \left[\norm{\nabla^2 f(\mathbf{x}, \xi)-\nabla^2 f(\mathbf{x})}^2 \right] \leq \sigma_2^2$
      and  $\norm{\nabla^2 f(\mathbf{x}, \xi)-\nabla^2 f(\mathbf{x})} \leq M_2$ a.s.
      \end{itemize}
     \label{def: A2}
   \end{assumption}
We note that the above assumptions are not essential to our result, and can be replaced by any conditions that give rise to concentration. Moreover, the a.s.\ bounded Hessian assumption can be removed if one further assumes $f(\mathbf{x}, \xi)$ has $\ell'$-Lipschitz gradients for all $\xi$, which is stronger than Assumption \ref{def: A1}. Assuming $f(\mathbf{x}, \xi)$ has $\ell'$-Lipschitz gradients immediately implies the stochastic Hessians are a.s.\ bounded with parameter $2\ell'$ which alone is sufficient to guarantee concentration without any further assumptions on the variance of $\nabla^2 f(\mathbf{x}, \xi)$ \citep{tropp2015introduction}.

\subsection{Cubic Regularization}
Our target in this paper is to find an $\epsilon$-second-order stationary point, which we define as follows:
\begin{definition} \label{def: eps-2}
  For a $\rho$-Hessian Lipschitz function $f$, we say that $\mathbf{x}$ is an \textit{$\epsilon$-second-order stationary point} (or \textit{$\epsilon$-approximate local minimum}) if
  \begin{align}
    \norm{\nabla f(\mathbf{x})} \leq \epsilon \quad \text{and} \quad \lambda_{\min}(\nabla^2 f(\mathbf{x})) \geq -\sqrt{\rho \epsilon}.
  \end{align}
\end{definition}
An $\epsilon$-second-order stationary point not only has a small gradient, but also has a Hessian which is close to positive semi-definite. Thus it is often also referred to as an $\epsilon$-approximate local minimum.

In the deterministic setting, cubic regularization \citep{nesterov2006cubic} is a classic algorithm for finding a second-order stationary point of a $\rho$-Hessian-Lipschitz function $f(\mathbf{x})$.
In each iteration, it first forms a local upper bound on the function using a third-order Taylor expansion around the current iterate $\mathbf{x}_t$:
\begin{align*}
m_t(\mathbf{x}) &= f(\mathbf{x}_t) + \nabla f(\mathbf{x}_t)^\top (\mathbf{x} - \mathbf{x}_t) + \frac{1}{2} (\mathbf{x} - \mathbf{x}_t)^\top \nabla^2 f(\mathbf{x}_t) (\mathbf{x} - \mathbf{x}_t) + \frac{\rho}{6} \norm{\mathbf{\mathbf{x} - \mathbf{x}_t}}^3 .
\label{eq:exact_cubic}
\end{align*}
This is called the \textit{cubic submodel}. Then, cubic regularization minimizes this cubic submodel to obtain the next iterate: $\mathbf{x}_{t+1} = \argmin_{\mathbf{x}} m_t(\mathbf{x})$. When the cubic submodel can be solved exactly, cubic regularization requires $\O\left(\frac{\sqrt{\rho} (f(\mathbf{x}_0)-f^*)}{\epsilon^{1.5}}\right)$ iterations to find an $\epsilon$-second-order stationary point.

To apply this algorithm in the stochastic setting, three issues need to be addressed: (1) we only have access to stochastic gradients and Hessians, not the true gradient and Hessian; (2) our only means of interaction with the Hessian is through Hessian-vector products; (3) the cubic submodel cannot be solved exactly in practice, only up to some tolerance. We discuss how to overcome each of these obstacles in our paper.

\section{Main Results}\label{sec:results}

\begin{algorithm}[!h]
\caption{Stochastic Cubic Regularization (Meta-algorithm)}\label{algo:SCRN}
\begin{algorithmic}[1]
\REQUIRE mini-batch sizes $n_1, n_2$, initialization $\mathbf{x_{0}}$, number of iterations $T_{\text{out}}$, and final tolerance $\epsilon$.
 \FOR{$t = 0, \hdots, T_{\text{out}}$}
 \STATE Sample $S_1 \leftarrow \{\mathbf{\xi}_{i} \}_{i=1}^{n_1}$, $S_2 \leftarrow \{\mathbf{\xi}_{i} \}_{i=1}^{n_2}$.
 \STATE $\mathbf{g}_t \leftarrow \frac{1}{|S_1|} \sum_{\mathbf{\xi}_i \in S_1} \nabla f(\mathbf{x}_{t}; \ \mathbf{\xi}_i)$
 \STATE $\mathbf{B}_t[\cdot] \leftarrow \frac{1}{|S_2|} \sum_{\xi_i\in S_2} \nabla^2 f(\mathbf{x}_t, \xi_i)(\cdot)$
 \STATE $\mathbf{\Delta}$, $\Delta_m \leftarrow \text{Cubic-Subsolver}(\mathbf{g}_t, \ \mathbf{B}_t[\cdot], \epsilon)$
 \STATE $\mathbf{x}_{t+1} \leftarrow \mathbf{x}_{t} + \mathbf{\Delta}$
 \IF{ $\Delta_m \geq -\frac{1}{100} \sqrt{\frac{\epsilon^3}{\rho}}$}  \label{algo1:line:small_descent}
  \STATE $\mathbf{\Delta} \leftarrow \text{Cubic-Finalsolver}(\mathbf{g}_t, \ \mathbf{B}_t[\cdot], \epsilon)$
  \STATE $\mathbf{x}^* \leftarrow \mathbf{x}_t + \mathbf{\Delta}$
  \BREAK
 \ENDIF
 \ENDFOR
\ENSURE $\mathbf{x}^*$ if the early termination condition was reached, otherwise the final iterate $x_{T_{\text{out}} + 1}$.
\end{algorithmic}
\end{algorithm}

We begin with a general-purpose stochastic cubic regularization meta-algorithm in Algorithm \ref{algo:SCRN}, which employs a black-box subroutine to solve stochastic cubic subproblems. At a high level, in order to deal with stochastic gradients and Hessians, we sample two independent minibatches $S_1$ and $S_2$ at each iteration. Denoting the average gradient by
\begin{align}
  \mathbf{g}_t = \frac{1}{|S_1|} \sum_{\xi_i \in S_1} \nabla f(\mathbf{x}_t, \xi_i)
\end{align}
and the average Hessian by
\begin{align}
  \mathbf{B}_t = \frac{1}{|S_2|} \sum_{\xi_i \in S_2} \nabla^2 f (\mathbf{x}_t, \xi_i),
\end{align}
this implies a \textit{stochastic cubic submodel}:
\begin{align}
m_t(\mathbf{x}) &= f(\mathbf{x}_t) + (\mathbf{x}-\mathbf{x}_t)^\top \mathbf{g}_t + \frac{1}{2} (\mathbf{x}-\mathbf{x}_t)^\top \mathbf{B}_t (\mathbf{x}-\mathbf{x}_t) + \frac{\rho}{6} \norm{\mathbf{x}-\mathbf{x}_t}^3 \label{eq:approx_cubic} .
\end{align}
Although the subproblem depends on $\mathbf{B}_t$, we note that our meta-algorithm never explicitly formulates this matrix, only providing the subsolver access to $\mathbf{B}_t$ through Hessian-vector products, which we denote by $\mathbf{B}_t[\cdot] : \mathbb{R}^d \to \mathbb{R}^d$. We hence assume that the subsolver performs gradient-based optimization to solve the subproblem, as $\nabla m_t(\mathbf{x})$ depends on $\mathbf{B}_t$ only via $\mathbf{B}_t[\mathbf{x}-\mathbf{x}_t]$.

After sampling minibatches for the gradient and the Hessian, Algorithm \ref{algo:SCRN} makes a call to a black-box cubic subsolver to optimize the stochastic submodel $m_t(\x)$. The subsolver returns a parameter change $\bDelta$, i.e., an approximate minimizer of the submodel, along with the corresponding change in submodel value, $\Delta_m \coloneqq m_t(\mathbf{x}_t + \bDelta) - m_t(\mathbf{x}_t)$. The algorithm then updates the parameters by adding $\bDelta$ to the current iterate, and checks whether $\Delta_m$ satisfies a stopping condition.

In more detail, the Cubic-Subsolver subroutine takes a vector $\mathbf{g}$ and a function for computing Hessian-vector products $\mathbf{B}[\cdot]$, then optimizes the third-order polynomial
\begin{equation}\label{eq:submodel_short}
\tilde{m}(\bDelta) = \bDelta^\top \mathbf{g} + \frac{1}{2}\bDelta^\top \mathbf{B}[\bDelta] + \frac{\rho}{6} \norm{\bDelta}^3 .
\end{equation}
Let $\bDelta^\star = \argmin_{\bDelta} \tilde{m}(\bDelta)$ denote the minimizer of this polynomial. In general, the subsolver cannot return the exact solution $\bDelta^\star$. We hence tolerate a certain amount of suboptimality:

\begin{samepage}
\begin{condition}
\label{cond:cubic_subsolver} 
For any fixed, small constant $c$, $\text{Cubic-Subsolver}(\mathbf{g}, \mathbf{B}[\cdot], \epsilon)$ terminates within $\mathcal{T}(\epsilon)$ gradient iterations (which may depend on $c$), and returns a $\mathbf{\Delta}$ satisfying at least one of the following:
\begin{enumerate}
\item $\max\{\tilde{m}(\bDelta), f(\x_t + \bDelta) - f(\x_t)\} \le -\Omega(\sqrt{\epsilon^3/\rho})$.  (\textbf{Case 1}) \label{cond:oracle}
\item $\norm{\mathbf{\Delta}} \leq \norm{\mathbf{\Delta}^\star} + c \sqrt{\frac{\epsilon}{\rho}}$ and, if $\norm{\bDelta^\star} \ge \frac{1}{2} \sqrt{\epsilon/\rho}$, then
$\tilde{m}(\bDelta) \le \tilde{m}(\bDelta^\star) + \frac{c}{12} \cdot \rho\norm{\bDelta^\star}^3$. (\textbf{Case 2})  \label{cond:approx_min}
\end{enumerate}
\end{condition}
\end{samepage}

The first condition is satisfied if the parameter change $\bDelta$ results in submodel and function decreases that are both sufficiently large (\textbf{Case 1}). If that fails to hold, the second condition ensures that $\bDelta$ is not too large relative to the true solution $\bDelta^\star$, and that the cubic submodel is solved to precision $c \cdot \rho \norm{\bDelta^\star}^3$ when $\norm{\bDelta^\star}$ is large (\textbf{Case 2}).

As mentioned above, we assume the subsolver uses gradient-based optimization to solve the subproblem so that it will only access the Hessian only through Hessian-vector products.
Accordingly, it can be any standard first-order algorithm such as gradient descent, Nesterov's accelerated gradient descent, etc. Gradient descent is of particular interest as it can be shown to satisfy Condition \ref{cond:cubic_subsolver} (see Lemma \ref{lem:subsolveerr}).

Having given an overview of our meta-algorithm and verified the existence of a suitable subsolver, we are ready to present our main theorem:

\begin{restatable}{theorem}{thmmainabstract}
 \label{thm:main1}
  There exists an absolute constant $c$ such that if
  $f(\mathbf{x})$ satisfies Assumptions \ref{def: A1}, \ref{def: A2}, CubicSubsolver satisfies Condition \ref{cond:cubic_subsolver} with $c$, $n_1 \geq \max(\frac{M_1}{c \epsilon}, \frac{\sigma_1^2}{c^2 \epsilon^2}) \log(\frac{d \sqrt{\rho} \Delta_f}{\epsilon^{1.5} \delta c})$, and $n_2 \geq \max(\frac{M_2}{c \sqrt{\rho \epsilon}}, \frac{\sigma_2^2}{c^2 \rho \epsilon}) \log( \frac{d\sqrt{\rho} \Delta_f}{\epsilon^{1.5} \delta c})$,
  then for all $\delta > 0$ and $\Delta_f \geq f(\mathbf{x}_0)-f^*$, 
  Algorithm \ref{algo:SCRN} will output an $\epsilon$-second-order stationary point of $f$ with probability at least $1-\delta$ within
  \begin{align}
    \mathcal{\tilde{O}}\left(\frac{\sqrt{\rho} \Delta_f}{\epsilon^{1.5}} \left( \max \left(\frac{M_1}{\epsilon}, \frac{\sigma_1^2}{\epsilon^2}\right) + \max\left(\frac{M_2}{\sqrt{\rho \epsilon}}, \frac{\sigma_2^2}{\rho \epsilon}\right) \cdot \mathcal{T}(\epsilon) \right) \right)
  \end{align}
  total stochastic gradient and Hessian-vector product evaluations.
\end{restatable}
In the limit where $\epsilon$ is small the result simplifies:
\begin{remark}
    \label{rmk:main1}
If $\epsilon \leq \min \left \{ \frac{\sigma_1^2}{c_1 M_1}, \frac{\sigma_2^4}{c_2^2 M_2^2 \rho}  \right \}$, then under the settings of Theorem~\ref{thm:main1} we can conclude that Algorithm~\ref{algo:SCRN} will output an $\epsilon$-second-order stationary point of $f$ with probability at least $1-\delta$ within
  \begin{align}
    \mathcal{\tilde{O}}\left(\frac{\sqrt{\rho} \Delta_f}{\epsilon^{1.5}} \left(  \frac{\sigma_1^2}{\epsilon^2} + \frac{\sigma_2^2}{\rho \epsilon} \cdot \mathcal{T}(\epsilon) \right) \right)
  \end{align}
  total stochastic gradient and Hessian-vector product evaluations.
\end{remark}
Theorem \ref{thm:main1} states that after $T_{\text{out}} = \O( \frac{\sqrt{\rho}(f(\mathbf{x}_0)-f^*)}{\epsilon^{1.5}})$ iterations, stochastic cubic regularization (Algorithm~\ref{algo:SCRN}) will have found an $\epsilon$-second-order stationary point w.h.p. -- see Equation \eqref{eq:outer_iter_complexity} in the Appendix for an exact outer iteration complexity that guarantees early termination of Algorithm~\ref{algo:SCRN} w.h.p. Within each iteration, we require $n_1 = \tilde{\mathcal{O}}(\frac{\sigma_1^2}{\epsilon^2})$ samples for gradient averaging and $n_2 = \tilde{\mathcal{O}}(\frac{\sigma_2^2}{\rho\epsilon})$ samples for Hessian-vector product averaging to ensure small variance when $\epsilon$ is small. Recall that the averaged gradient $\mathbf{g}_t$ is fixed for a given cubic submodel, while the averaged Hessian-vector product $\mathbf{B}_t[\mathbf{v}]$ needs to be recalculated every time the cubic subsolver queries the gradient $\nabla \tilde{m}(\cdot)$. At most $\mathcal{T}(\epsilon)$ such queries will be made by definition. Therefore, each iteration takes $\O(\frac{\sigma_1^2}{\epsilon^2} + \frac{\sigma_2^2}{\rho \epsilon} \cdot \mathcal{T}(\epsilon)) $ stochastic gradient and Hessian-vector product evaluations in the limit where $\epsilon$ is small (see Remark \ref{rmk:main1}). Multiplying these two quantities together gives the full result.

\begin{algorithm}[t]
\caption{Cubic-Finalsolver via Gradient Descent}\label{algo:SCD-f}
\begin{algorithmic}[1]
\REQUIRE $\mathbf{g}$, $\mathbf{B}[\cdot]$, tolerance $\epsilon$.
 \STATE $\mathbf{\Delta} \leftarrow 0$, $\mathbf{g}_m \leftarrow \mathbf{g}$, $\eta \leftarrow \frac{1}{20 \ell}$
 \WHILE{$ \norm{\mathbf{g}_m} > \frac{\epsilon}{2}$ }
  \STATE $\mathbf{\Delta} \leftarrow \mathbf{\Delta} - \eta \mathbf{g}_m$
  \STATE $\mathbf{g}_m \leftarrow \mathbf{g} + \mathbf{B}[\mathbf{\Delta}] + \frac{\rho}{2} \norm{\mathbf{\Delta}} \mathbf{\Delta}$
 \ENDWHILE
 \ENSURE $\mathbf{\Delta}$
\end{algorithmic}
\end{algorithm}

Finally, we note that lines 8-11 of Algorithm \ref{algo:SCRN} give the termination condition of our meta-algorithm. When the decrease in submodel value $\Delta_m$ is too small, our theory guarantees $\x_t + \bDelta^\star$ is an $\epsilon$-second-order stationary point, where $\bDelta^\star$ is the optimal solution of the cubic submodel. However, Cubic-Subsolver may only produce an inexact solution $\bDelta$ satisfying Condition \ref{cond:cubic_subsolver}, which is not sufficient for $\x_t + \bDelta$ to be an $\epsilon$-second-order stationary point. We therefore call Cubic-Finalsolver to solve the subproblem with higher precision. Since Cubic-Finalsolver is invoked only once at the end of the algorithm, we can just use gradient descent, and its runtime will always be dominated by the rest of the algorithm.

\subsection{Gradient Descent as a Cubic-Subsolver}

One concrete example of a cubic subsolver is a simple variant of gradient descent (Algorithm \ref{algo:SCD}) as studied in \citet{carmon2016gradient}. The two main differences relative to standard gradient descent are: (1) lines 1--3: when $\mathbf{g}$ is large, the submodel (Equation \ref{eq:submodel_short}) may be ill-conditioned, so instead of doing gradient descent, the iterate only moves one step in the $\mathbf{g}$ direction, which already guarantees sufficient descent; (2) line 6: the algorithm adds a small perturbation to $\mathbf{g}$ to avoid a certain ``hard" case for the cubic submodel. We refer readers to \citet{carmon2016gradient} for more details about Algorithm \ref{algo:SCD}.

Adapting their result for our setting, we obtain the following lemma, which states that the gradient descent subsolver satisfies our Condition \ref{cond:cubic_subsolver}.
\begin{restatable}{lemma}{subsolveerr}\label{lem:subsolveerr}
There exists an absolute constant $c'$, such that under the same assumptions on $f(\mathbf{x})$ and the same choice of parameters $n_1, n_2$ as in Theorem \ref{thm:main1}, Algorithm \ref{algo:SCD} satisfies Condition \ref{cond:cubic_subsolver} with probability at least $1-\delta'$ with
\begin{align}
  \mathcal{T}(\epsilon) \le \tilde{\O}(\frac{\ell}{\sqrt{\rho\epsilon}}) .
\end{align}
\end{restatable}

\begin{algorithm}[t]
\caption{Cubic-Subsolver via Gradient Descent}\label{algo:SCD}
\begin{algorithmic}[1]
\REQUIRE $\mathbf{g}$, $\mathbf{B}[\cdot]$, tolerance $\epsilon$.
 \IF{$\norm{\mathbf{g}} \geq \frac{\ell^2}{\rho}$}{ \label{algo2:line:big_grad}
  \STATE $R_c \leftarrow -\frac{\mathbf{g}^\top \mathbf{B}[\mathbf{g}]}{\rho \norm{\mathbf{g}}^2} + \sqrt{\left( \frac{\mathbf{g}^\top \mathbf{B} [\mathbf{g}]}{\rho \norm{\mathbf{g}}^2} \right)^2 + \frac{2 \norm{\mathbf{g}}}{\rho}}$
   \STATE  $\mathbf{\Delta} \leftarrow -R_c \frac{\mathbf{g}}{\norm{\mathbf{g}}}$ \label{algo2:line:cauchy}
   }
 \ELSE{
  \STATE $\mathbf{\Delta} \leftarrow 0$, $\sigma \leftarrow c' \frac{\sqrt{\epsilon\rho}}{\ell}$, $\eta \leftarrow \frac{1}{20 \ell}$
  \STATE $\tilde{\mathbf{g}} \leftarrow \mathbf{g} + \sigma \mathbf{\zeta}$ for $\mathbf{\zeta} \sim \text{Unif}(\mathbb{S}^{d-1})$
 \FOR{$t = 1, \hdots, \mathcal{T}(\epsilon)$} \label{algo2:line:grad_loop}
  \STATE $\mathbf{\Delta} \leftarrow \bDelta  - \eta (\tilde{\mathbf{g}} + \eta \mathbf{B}[\bDelta] + \frac{\rho}{2} \norm{\mathbf{\Delta}} \mathbf{\Delta})$
 \ENDFOR
 }
 \ENDIF
 \STATE $\Delta_m \leftarrow \mathbf{g}^\top \mathbf{\Delta} + \frac{1}{2} \mathbf{\Delta}^\top \mathbf{B}[\bDelta] + \frac{\rho}{6} \norm{\mathbf{\Delta}}^3$
 \ENSURE $\mathbf{\Delta}$, $\Delta_m$
\end{algorithmic}
\end{algorithm}

Our next corollary applies gradient descent (Algorithm \ref{algo:SCD}) as the approximate cubic subsolver in our meta-algorithm (Algorithm \ref{algo:SCRN}), which immediately gives the total number of gradient and Hessian-vector evaluations for the full algorithm.
 \begin{restatable}{corollary}{thmmaincarmon}
  \label{thm:main2}
  Under the same settings as Theorem \ref{thm:main1}, if $\epsilon \leq \min \left \{ \frac{\sigma_1^2}{c_1 M_1}, \frac{\sigma_2^4}{c_2^2 M_2^2 \rho}  \right \}$, and if we instantiate the Cubic-Subsolver subroutine with Algorithm \ref{algo:SCD}, 
  then with probability greater than $1-\delta$, Algorithm \ref{algo:SCRN}
   will output an $\epsilon$-second-order stationary point of $f(\mathbf{x})$ within
   \begin{align}
    \mathcal{\tilde{O}}\left(\frac{\sqrt{\rho} \Delta_f}{\epsilon^{1.5}} \left( \frac{\sigma_1^2}{\epsilon^2} + \frac{\sigma_2^2}{\rho \epsilon} \cdot \frac{\ell}{\sqrt{\rho\epsilon}}\right) \right)
   \end{align}
    total stochastic gradient and Hessian-vector product evaluations.
 \end{restatable}

\begin{remark}
The overall runtime of Algorithm \ref{algo:SCRN} in Corollary \ref{thm:main2} is the number of stochastic gradient and Hessian-vector product evaluations multiplied by the time to compute a gradient or Hessian-vector product. For neural networks, the latter takes $\O(d)$ time.
\end{remark}

From Corollary \ref{thm:main2}, we observe that the dominant term in solving the submodel is $\frac{\sigma_1^2}{\epsilon^2}$ when $\epsilon$ is sufficiently small, giving a total iteration complexity of $\O(\epsilon^{-3.5})$ when other problem-dependent parameters are constant. This improves on the $\O(\epsilon^{-4}\text{poly}(d))$ complexity attained by SGD.

It is reasonable to believe there may be another cubic subsolver which is faster than gradient descent and which satisfies Condition \ref{cond:cubic_subsolver} with a smaller $\mathcal{T}(\epsilon)$, for instance a variant of Nesterov's accelerated gradient descent. However, since the dominating term in our subsolver complexity is $\frac{\sigma_1^2}{\epsilon^2}$ due to gradient averaging, and this is independent of $\mathcal{T}(\epsilon)$, a faster cubic subsolver cannot improve the overall number of gradient and Hessian-vector product evaluations. This means that the gradient descent subsolver already achieves the optimal asymptotic rate for finding an $\epsilon$-second-order stationary point under our stochastic cubic regularization framework.

\section{Proof Sketch} \label{sec:proof_sketch}
This section sketches the crucial steps needed to understand and prove our main theorem (Theorem \ref{thm:main1}).  We begin by describing our high-level approach, then show how to instantiate this high-level approach in the stochastic setting, assuming oracle access to an exact subsolver. For the case of an inexact subsolver and other proof details, we defer to the Appendix.

Recall that at iteration $t$ of Algorithm \ref{algo:SCRN}, a stochastic cubic submodel $m_t$ is constructed around the current iterate $\x_t$ with the form given in
Equation \eqref{eq:approx_cubic}:
\begin{align*}
m_t(\mathbf{x}) &= f(\mathbf{x}_t) + (\mathbf{x}-\mathbf{x}_t)^\top \mathbf{g}_t + \frac{1}{2} (\mathbf{x}-\mathbf{x}_t)^\top \mathbf{B}_t (\mathbf{x}-\mathbf{x}_t) + \frac{\rho}{6} \norm{\mathbf{x}-\mathbf{x}_t}^3 ,
\end{align*}
where $\mathbf{g}_t$ and $\mathbf{B}_t$ are the averaged stochastic gradients and Hessians.
At a high level, we will show that for each iteration,
the following two claims hold:

\textbf{Claim 1.} If $\x_{t+1}$ is not an $\epsilon$-second-order stationary point of $f(\mathbf{x})$, the cubic submodel has large descent $m_t(\x_{t+1}) - m_t(\x_t)$.

\textbf{Claim 2.} If the cubic submodel has large descent $m_t(\x_{t+1}) - m_t(\x_t)$, then the true function also has large descent $f(\x_{t+1}) - f(\x_t)$.

Given these claims, it is straightforward to argue for the correctness of our approach. We know that if we observe a large decrease in the cubic submodel value $m_t(\x_{t+1}) - m_t(\x_t)$ during Algorithm~\ref{algo:SCRN}, then by Claim~2 the function will also have large descent. But since $f$ is bounded below, this cannot happen indefinitely, so we must eventually encounter an iteration with small cubic submodel descent. When that happens, we can conclude by Claim~1 that $\x_{t+1}$ is an $\epsilon$-second-order stationary point.

We note that Claim~2 is especially important in the stochastic setting, as we no longer have access to the true function but only the submodel. Claim~2 ensures that progress in $m_t$ still indicates progress in $f$, allowing the algorithm to terminate at the correct time.

In the remaining parts of this section, we discuss why the above two claims hold for an exact solver.

\subsection{Stochastic Setting with Exact Subsolver} \label{sec:SCRN}
In this setting, $\mathbf{g}_t$ and $\mathbf{B}_t$ are the averaged gradient and Hessian with sample sizes $n_1$ and $n_2$, respectively.
To ensure the stochastic cubic submodel approximates the exact cubic submodel well, we need large enough sample sizes so that both
$\mathbf{g}_t$ and $\mathbf{B}_t$ are close to the exact gradient and Hessian at $\x_t$ up to some tolerance:
\begin{restatable}{lemma}{concconds}
  \label{lem:conc_conds}
  For any fixed small constants $c_1, c_2$, we can pick gradient and Hessian mini-batch sizes $n_1 = \tilde{\O} \left(\max \left(\frac{M_1}{\epsilon}, \frac{\sigma_1^2}{\epsilon^2}\right) \right)$ and $n_2 = \tilde{\O}\left(\max \left(\frac{M_2}{\sqrt{\rho \epsilon}}, \frac{\sigma_2^2}{\rho \epsilon} \right) \right)$ so that with probability $1-\delta'$,
\begin{equation}
  \|\mathbf{g}_t -  \nabla f(\mathbf{x}_t) \| \leq c_1 \cdot \epsilon , \label{gradconc}
\end{equation}
\begin{equation}
  \forall \mathbf{v},  \| (\mathbf{B}_t-\nabla^2 f(\mathbf{x}_t))\mathbf{v} \| \leq c_2 \cdot \sqrt{\rho \epsilon}\norm{\mathbf{v}} .
  \label{hessconc}
\end{equation}
\end{restatable}
\noindent
We need to ensure that the random vectors/matrices concentrate along an arbitrary direction (depending on $\mathbf{g}_t$ and $\mathbf{B}_t$). In order to guarantee the uniform concentration in Lemma \ref{lem:conc_conds}, we can directly apply results from matrix concentration to obtain the desired result \citep{tropp2015introduction}.

Let $\bDelta_t^\star = \argmin_\bDelta m_t(\x_t + \bDelta)$, i.e.\ $\x_t + \bDelta_t^\star$ is a global minimizer of the cubic submodel $m_t$. If we use an exact oracle solver, we have $\x_{t+1} = \x_t + \bDelta_t^\star$. In order to show Claim 1 and Claim 2, one important quantity to study is the decrease in the cubic submodel $m_t$:
      \begin{restatable}{lemma}{optdescent}\label{lem:optdescent}
      Let $m_t$ and $\bDelta_t^\star$ be defined as above. Then for all $t$,
        $$m_t(\x_t + \bDelta_t^\star) - m_t(\x_t) \le - \frac{1}{12} \rho \norm{\mathbf{\Delta}_t^\star}^3 .$$
      \end{restatable}
Lemma \ref{lem:optdescent} implies that in order to prove submodel $m_t$ has sufficient function value decrease, we only need to lower bound the norm of optimal solution, i.e.\ $\norm{\bDelta_t^\star}$.

\textbf{Proof sketch for claim 1:} Our strategy is to lower bound the norm of $\bDelta_t^\star$ when $\x_{t+1} = \x_t + \bDelta_t^\star$ is not an $\epsilon$-second-order stationary point. In the non-stochastic setting, \citet{nesterov2006cubic} prove
$$\norm{\mathbf{\Delta}_t^\star} \geq \frac{1}{2} \max \Bigg \{\sqrt{\frac{1}{\rho}\norm{\nabla f(\mathbf{x}_{t+1})}},   \frac{1}{\rho} \lambda_{\min}(\nabla^2 f(\mathbf{x}_{t+1}))  \Bigg\}, \label{eq:norm_lowerbound}$$
which gives the desired result.
In the stochastic setting, a similar statement holds up to some tolerance:
\begin{restatable}{lemma}{sbig}
  \label{lem:s_big}
    Under the setting of Lemma \ref{lem:conc_conds} with sufficiently small constants $c_1, c_2$,
    \begin{align}
    \norm{\mathbf{\Delta}^\star_t} \geq \frac{1}{2} \max \Bigg \{&\sqrt{\frac{1}{\rho}\left(\norm{\nabla f(\mathbf{x}_t + \mathbf{\Delta}^\star_t)}-\frac{\epsilon}{4}\right)}, \frac{1}{\rho} \left(\lambda_{\min}(\nabla^2 f(\mathbf{x}_t+\mathbf{\Delta}_t^\star)) - \frac{\sqrt{\rho\epsilon}}{4}\right) \Bigg\}. \notag
    \end{align}
\end{restatable}
That is, when $\x_{t+1}$ is not an $\epsilon$-second-order stationary point, we have $\norm{\mathbf{\Delta}_t^\star} \geq \Omega(\sqrt{\frac{\epsilon}{\rho}})$. In other words, we have sufficient movement.
It follows by Lemma~\ref{lem:optdescent} that we have sufficient cubic submodel descent.

\textbf{Proof sketch for claim 2:} In the non-stochastic case, $m_t(\x)$ is by construction an upper bound on $f(\x)$. Together with the fact $f(\x_t) = m_t(\x_t)$, we have:
$$f(\x_{t+1}) - f(\x_t) \le m_t(\x_{t+1}) - m_t(\x_t) ,$$
showing Claim 2 is always true. For the stochastic case, this inequality may no longer be true. Instead, under the setting of Lemma \ref{lem:conc_conds}, via Lemma \ref{lem:optdescent}, we can upper bound the function decrease with an additional error term:
\begin{align}
    f(\x_{t+1}) -& f(\x_t) \le \frac{1}{2}[m_t(\x_{t+1}) - m_t(\x_t)] + c \sqrt{\frac{\epsilon^3}{\rho}},\nonumber
\end{align}
for some sufficiently small constant $c$. Then when $m_t(\x_{t+1}) - m_t(\x_t) \le -4c\sqrt{\epsilon^3/\rho}$, we have
$f(\x_{t+1}) - f(\x_t) \le \frac{1}{4}[m_t(\x_{t+1}) - m_t(\x_t)] \le -c\sqrt{\epsilon^3/\rho} $, which proves Claim 2.

Finally, for an approximate cubic subsolver, the story becomes more elaborate. Claim 1 is only ``approximately'' true, while Claim 2 still holds but for more complicated reasons. We defer to the Appendix for the full proof.

\section{Experiments}\label{sec:experiments}

In this section, we provide empirical results on synthetic and real-world data sets to demonstrate the efficacy of our approach. All experiments are implemented using TensorFlow~\citep{abadi2016tensorflow}, which allows for efficient computation of Hessian-vector products using the method described by \citet{pearlmutter1994fast}.

\subsection{Synthetic Nonconvex Problem}

\begin{figure}[t]
\centering
\includegraphics[width=0.45\textwidth]{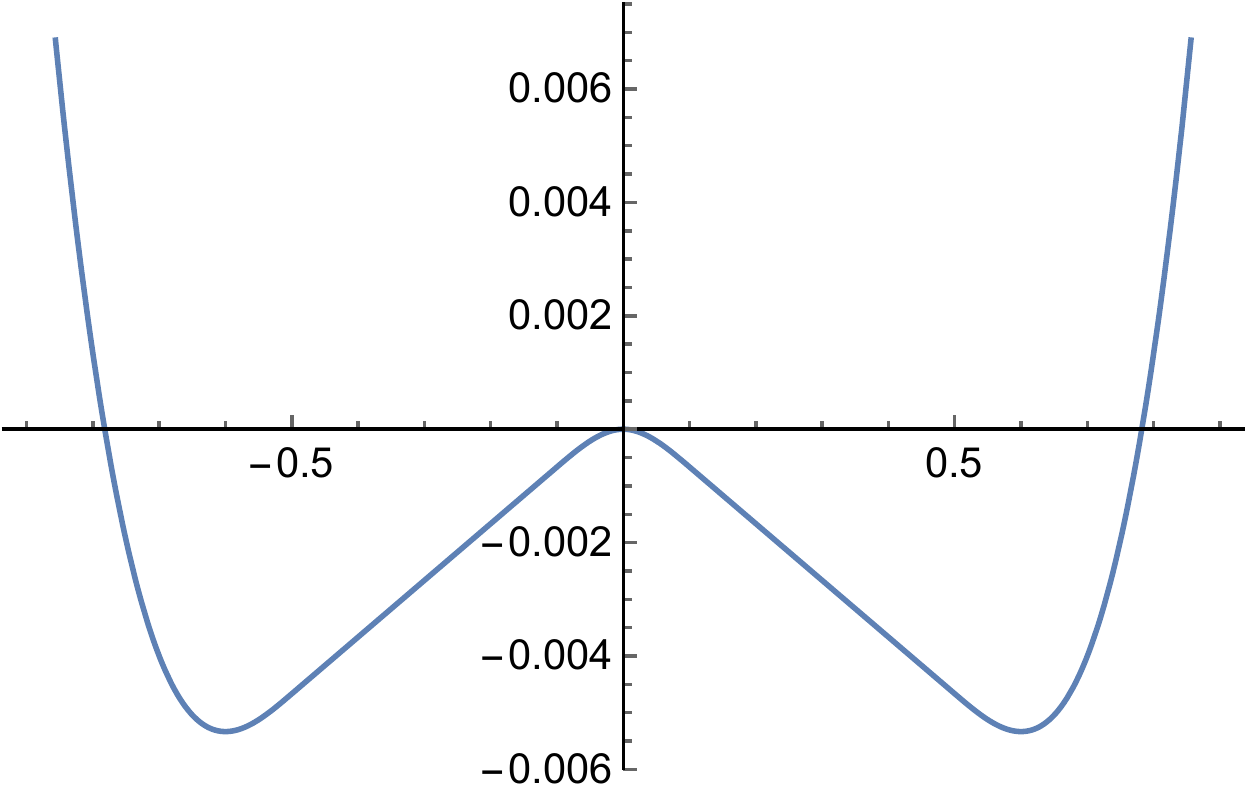}
\caption{The piecewise cubic function $w(x)$ used along one of the dimensions in the synthetic experiment. The other dimension uses a scaled quadratic.}
\label{fig:piecewise-plot}
\end{figure}

We begin by constructing a nonconvex problem with a saddle point to compare our proposed approach against stochastic gradient descent. Let $w(x)$ be the W-shaped scalar function depicted in Figure~\ref{fig:piecewise-plot}, with a local maximum at the origin and two local minima on either side. While we defer the exact form of $w(x)$ to Appendix~\ref{apx:experimental-details}, we note here that it has small negative curvature at the origin, $w''(0) = -0.2$, and that it has a 1-Lipschitz second derivative. We aim to solve the problem $$\min_{\mathbf{x} \in \mathbb{R}^2} \left[ w(x_1) + 10 x_2^2 \right] ,$$ with independent noise drawn from $\mathcal{N}(0,1)$ added separately to each component of every gradient and Hessian-vector product evaluation. By construction, the objective function has a saddle point at the origin with Hessian eigenvalues of -0.2 and 20, providing a simple but challenging test case where the negative curvature is two orders of magnitude smaller than the positive curvature and is comparable in magnitude to the noise.

We ran our method and SGD on this problem, plotting the objective value versus the number of oracle calls in Figure~\ref{fig:synthetic-results}. The batch sizes and learning rates for each method are tuned separately to ensure a fair comparison; see Appendix~\ref{apx:experimental-details} for details. We observe that our method is able to escape the saddle point at the origin and converge to one of the global minima faster than SGD, offering empirical evidence in support of our method's theoretical advantage.

\subsection{Deep Autoencoder}

\begin{figure}[t]
\centering
\begin{minipage}[t]{0.49\textwidth}
  \centering
  \includegraphics[width=0.9\textwidth]{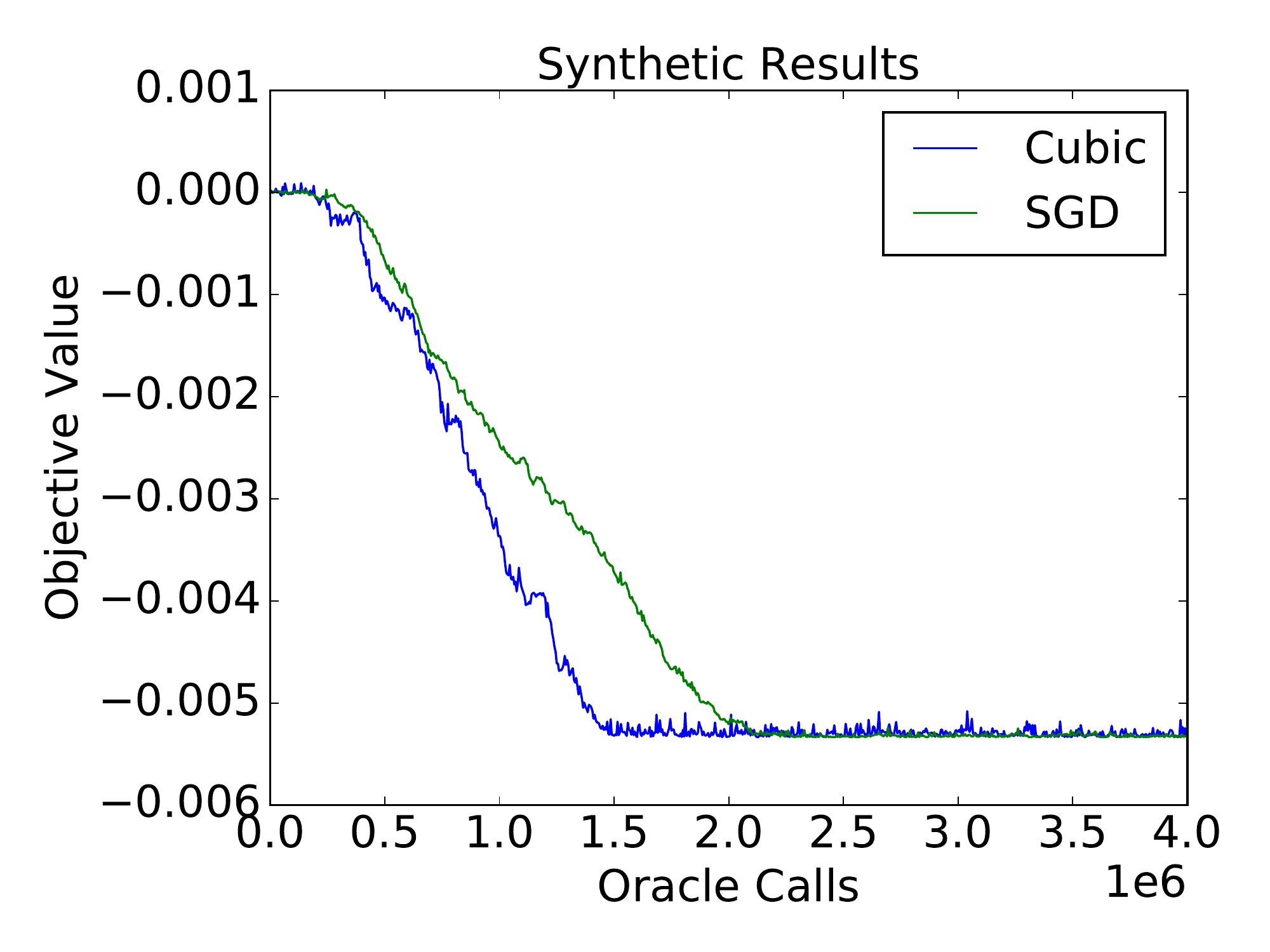}
  \caption{Results on our synthetic nonconvex optimization problem. Stochastic cubic regularization escapes the saddle point at the origin and converges to a global optimum faster than SGD.}
  \label{fig:synthetic-results}
\end{minipage}
\hfill
\begin{minipage}[t]{0.49\textwidth}
  \centering
  \includegraphics[width=0.9\textwidth]{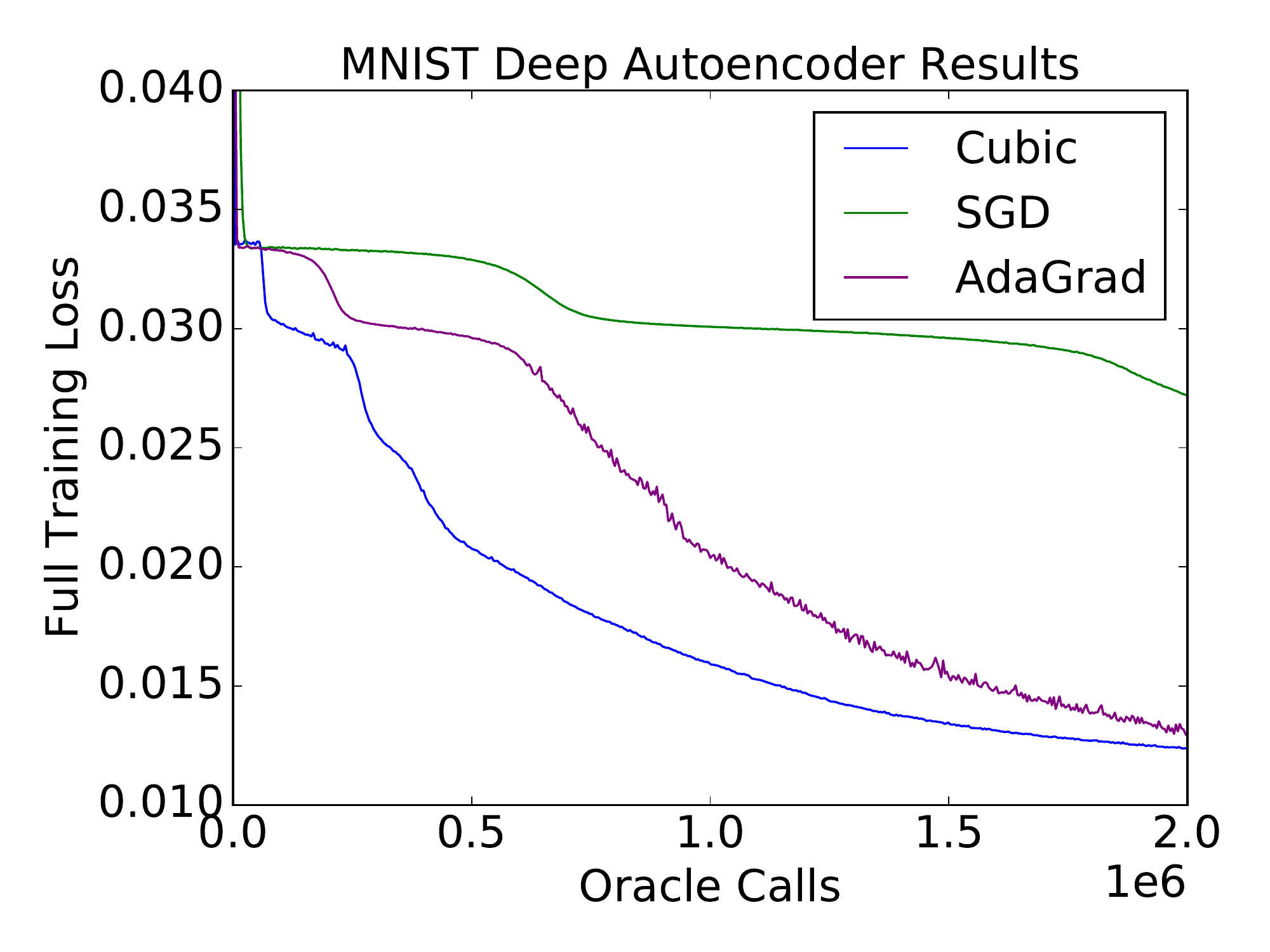}
  \caption{Results on the MNIST deep autoencoding task. Multiple saddle points are present in the optimization problem. Stochastic cubic regularization is able to escape them most quickly, allowing it to reach a local minimum faster than SGD and AdaGrad.}
  \label{fig:autoencoder-results}
\end{minipage}
\end{figure}

In addition to the synthetic problem above, we also investigate the performance of our approach on a more practical problem from deep learning, namely training a deep autoencoder on MNIST. The MNIST data set consists of grayscale images of handwritten digits, divided into 60,000 training images and 10,000 test images~\citep{lecun2010mnist}. Each image is $28 \times 28$ pixels, and pixel intensities are normalized to lie between 0 and 1. Our architecture consists of a fully connected encoder with dimensions $(28 \times 28) \to 512 \to 256 \to 128 \to 32$ together with a symmetric decoder. We use a softplus nonlinearity (defined as $\mathrm{softplus}(x) = \log(1 + \exp(x))$) for each hidden layer to ensure continuous first and second derivatives, and we use a pixelwise $\ell_2$ loss between the input and the reconstructed output as our objective function. We apply an elementwise sigmoid function to the final layer to guarantee that the output pixel intensities lie between 0 and 1.

Results on this autoencoding task are presented in Figure~\ref{fig:autoencoder-results}. In addition to training the model with our method and SGD, we also include results using AdaGrad, a popular adaptive first-order method with strong empirical performance~\citep{duchi2011adaptive}. Since the standard MNIST split does not include a validation set, we separate the original training set into 55,000 training images and 5,000 validation images, plotting training error on the former and using the latter to select hyperparameters for each method. More details about our experimental setup can be found in Appendix~\ref{apx:experimental-details}.

We observe that stochastic cubic regularization quickly escapes two saddle points and descends toward a local optimum, while AdaGrad takes two to three times longer to escape each saddle point, and SGD is slower still. This demonstrates that incorporating curvature information via Hessian-vector products can assist in escaping saddle points in practice. However, it is worth noting that AdaGrad makes slightly faster progress than our approach after reaching a basin around a local optimum, indicating that adaptivity may provide complementary benefits to second-order information. We leave the investigation of a hybrid method combining both techniques as an exciting direction for future work.


\section{Conclusion}\label{sec:concl}
In this paper, we presented a stochastic algorithm based on the classic cubic-regularized Newton method for nonconvex optimization.
Our algorithm provably finds  $\epsilon$-approximate local minima in $\mathcal{\tilde{O}}(\epsilon^{-3.5})$ total gradient and Hessian-vector product evaluations, improving upon the $\tilde{\O}(\epsilon^{-4})$ rate of SGD.
Our experiments show the favorable performance of our method relative to SGD on both a synthetic and a real-world problem.


\bibliographystyle{plainnat}
\bibliography{SHV}

\newpage
\onecolumn
\appendix
\section{Proof of Main Results} \label{sec:main_appendix}
In this section, we give formal proofs of Theorems \ref{thm:main1} and \ref{thm:main2}.
We start by providing proofs of several useful auxiliary lemmas.
\begin{remark}
  It suffices to assume that $\epsilon \leq \frac{\ell^2}{\rho}$ for the following analysis, since otherwise every point $\mathbf{x}$ satisfies the second-order condition
  $\lambda_{\min}(\nabla^2 f(\mathbf{x})) \geq -\sqrt{\rho \epsilon}$ trivially by the Lipschitz-gradient assumption.
\end{remark}

\subsection{Set-Up and Notation}

Here we remind the reader of the relevant notation and provide further background from \citet{nesterov2006cubic} on the cubic-regularized Newton method.
We denote the stochastic gradient as
\begin{align}
  \mathbf{g}_t = \frac{1}{|S_1|} \sum_{\xi_i \in S_1} \nabla f(\mathbf{x}_t, \xi_i) \notag
\end{align}
and the stochastic Hessian as
\begin{align}
  \mathbf{B}_t = \frac{1}{|S_2|} \sum_{\xi_i \in S_2} \nabla^2 f (\mathbf{x}_t, \xi_i), \notag
\end{align}
both for iteration $t$.
We draw a sufficient number of samples $|S_1|$ and $|S_2|$ so that the concentration conditions
\begin{equation}
  \|\mathbf{g}_t -  \nabla f(\mathbf{x}_t) \| \leq c_1 \cdot \epsilon , \notag
\end{equation}
\begin{equation}
  \forall \mathbf{v},  \| (\mathbf{B}_t-\nabla^2 f(\mathbf{x}_t))\mathbf{v} \| \leq c_2 \cdot \sqrt{\rho \epsilon}\norm{\mathbf{v}}. \notag
\end{equation}
are satisfied for sufficiently small $c_1, c_2$ (see Lemma \ref{lem:conc_conds} for more details).
The cubic-regularized Newton subproblem is to minimize
\begin{align}
m_t(\mathbf{y}) &= f(\mathbf{x}_t) + (\mathbf{y}-\mathbf{x}_t)^\top \mathbf{g}_t
                 + \frac{1}{2} (\mathbf{y}-\mathbf{x}_t)^\top \mathbf{B}_t (\mathbf{y}-\mathbf{x}_t) + \frac{\rho}{6} \norm{\mathbf{y}-\mathbf{x}_t}^3. \label{eq:approx_cubic_appendix}
\end{align}
We denote the global optimizer of $m_t(\cdot)$ as $\mathbf{x}_{t}+\mathbf{\Delta}_t^\star$, that is
$\mathbf{\Delta}_t^\star = \argmin_{z} m_{k}(\mathbf{z}+\mathbf{x}_k)$.

As shown in \citet{nesterov2006cubic} a global optima of Equation \eqref{eq:approx_cubic_appendix} satisfies:
\begin{align}
& \mathbf{g}_t + \mathbf{B}_t \mathbf{\Delta}_t^\star+ \frac{\rho}{2} \norm{\mathbf{\Delta}_t^\star} \mathbf{\Delta}_t^\star = 0. \label{eq:cubic_sol_stat} \\
& \mathbf{B}_t + \frac{\rho}{2} \norm{\mathbf{\Delta}_t^\star}I \succeq 0. \label{eq:cubic_sol_dual}
\end{align}
Equation \eqref{eq:cubic_sol_stat} is the first-order stationary condition, while Equation \eqref{eq:cubic_sol_dual} follows from a duality argument. In practice,
we will not be able to directly compute $\mathbf{\Delta}_t^\star$ so will instead use a Cubic-Subsolver routine which must satisfy:

\addtocounter{condition}{-1}
\begin{condition}
For any fixed, small constant $c_3, c_4$, $\text{Cubic-Subsolver}(\mathbf{g}, \mathbf{B}[\cdot], \epsilon)$ terminates within $\mathcal{T}(\epsilon)$ gradient iterations (which may depend on $c_3, c_4$), and returns a $\mathbf{\Delta}$ satisfying at least one of the following:
\begin{enumerate}
\item $\max\{\tilde{m}(\bDelta), f(\x_t + \bDelta) - f(\x_t)\} \le -\Omega(\sqrt{\epsilon^3/\rho})$.  (\textbf{Case 1}) \label{cond:oracle_append}
\item $\norm{\mathbf{\Delta}} \leq \norm{\mathbf{\Delta}^\star} + c_4 \sqrt{\frac{\epsilon}{\rho}}$ and, if $\norm{\bDelta^\star} \ge \frac{1}{2} \sqrt{\epsilon/\rho}$, then
$\tilde{m}(\bDelta) \le \tilde{m}(\bDelta^\star) + \frac{c_3}{12} \cdot \rho\norm{\bDelta^\star}^3$. (\textbf{Case 2})  \label{cond:approx_min_append}
\end{enumerate}
\end{condition}

\subsection{Auxiliary Lemmas}
We begin by providing the proof of several useful auxiliary lemmas.
First we provide the proof of Lemma \ref{lem:conc_conds} which characterize the concentration conditions.

\concconds*

\begin{proof}
We can use the matrix Bernstein inequality from \citet{tropp2015introduction} to control both the fluctuations in the stochastic gradients and stochastic Hessians under Assumption \ref{def: A2}.

Recall that the spectral norm of a vector is equivalent to its vector norm. So the matrix variance of the centered gradients $\tilde{\mathbf{g}} = \frac{1}{n_1} \sum_{i=1}^{n_1} \left( \tilde{\nabla} f(\mathbf{x}, \xi_i) \right) = \frac{1}{n_1} \sum_{i=1}^{n_1} \left( \nabla f(\mathbf{x}, \xi_i) - \nabla f(\mathbf{x}) \right)$ is: 
\begin{align}
v[\tilde{\mathbf{g}}] = \frac{1}{n_1^2} \max \left \{ \norm{ \mathbb{E} \left[\sum_{i=1}^{n_1} \tilde{\nabla} f(\mathbf{x}, \xi_i) \tilde{\nabla} f(\mathbf{x}, \xi_i)^\top \right]}, \norm{ \mathbb{E}\left[\sum_{i=1}^{n_1} \tilde{\nabla} f(\mathbf{x}, \xi_i)^\top \tilde{\nabla} f(\mathbf{x}, \xi_i)\right]}  \right \} \leq \frac{\sigma_1^2}{n_1} \notag
\end{align}
using the triangle inequality and Jensens inequality. A direct application of the matrix Bernstein inequality gives:
\begin{align}
& \mathbb{P} \left[\norm{\mathbf{g} - \nabla f(\mathbf{x})} \geq t \right] \leq 2d \exp(-\frac{t^2/2}{v[\tilde{\mathbf{g}}]+M_1/(3n_1)}) \leq 2d \exp(-\frac{3n_1}{8} \min \left \{ \frac{t}{M_1}, \frac{t^2}{\sigma_1^2} \right \}) \implies \notag \\
& \norm{\mathbf{g} - \nabla f(\mathbf{x})} \leq t \text{ with probability } 1-\delta' \text{ for } n_1 \geq \max \left( \frac{M_1}{t}, \frac{\sigma_1^2}{t^2} \right) \frac{8}{3} \log \frac{2d}{\delta'} \notag
\end{align}
Taking $t=c_1 \epsilon$ gives the result.
  
The matrix variance of the centered Hessians $\tilde{\mathbf{B}} = \frac{1}{n_2} \sum_{i=1}^{n_2} \left( \tilde{\nabla}^2 f(\mathbf{x}, \xi_i) \right) = \frac{1}{n_2} \sum_{i=1}^{n_2} \left( \nabla^2 f(\mathbf{x}, \xi_i) - \nabla^2 f(\mathbf{x}) \right)$, which are symmetric, is:
\begin{align}
    v[\tilde{\mathbf{B}}] = \frac{1}{n_2^2} \norm{ \sum_{i=1}^{n_2} \mathbb{E}\left[ \left(\tilde{\nabla}^2 f(\mathbf{x}, \xi_i)\right)^2 \right] }  \leq \frac{\sigma_2^2}{n_2}
\end{align}
once again using the triangle inequality and Jensens inequality. Another application of the matrix Bernstein inequality gives that:
  \begin{align}
    & \mathbb{P}[\norm{\mathbf{B} - \nabla^2 f(\mathbf{x}))} \geq t] \leq 2d \exp(-\frac{3n_2}{8} \min \{\frac{t}{M_2}, \frac{t^2}{\sigma_2^2} \}) \notag \implies \\
    & \norm{\mathbf{B} - \nabla^2 f(\mathbf{x}))}  \leq t \text{ with probability } 1-\delta' \text{ for } n_2 \geq \max( \frac{M_2}{t}, \frac{\sigma_2^2}{t^2}) \frac{8}{3} \log \frac{2d}{\delta'} \notag
  \end{align} 
  Taking $t=c_2 \sqrt{\rho \epsilon}$ ensures that the stochastic Hessian-vector products are controlled uniformly over $\mathbf{v}$:
  \begin{align}
      \norm{(\mathbf{B}-\nabla^2 f(\mathbf{x})) \mathbf{v}} \leq c_2 \cdot \sqrt{\rho \epsilon} \norm{\mathbf{v}} \notag
  \end{align}
  using $n_2$ samples with probability $1-\delta'$.
\end{proof}
Next we show Lemma \ref{lem:optdescent} which will relate the change in the cubic function value to the norm $\norm{\mathbf{\Delta}_t^\star}$.
\optdescent*
\begin{proof}
Using the global optimality conditions for Equation \eqref{eq:approx_cubic_appendix} from \citet{nesterov2006cubic}, we have the global optima $\mathbf{x}_t+\mathbf{\Delta}_t^\star$, satisfies:
\begin{align}
  & \mathbf{g}_t + \mathbf{B}_t (\mathbf{\Delta}_t^\star) + \frac{\rho}{2} \norm{\mathbf{\Delta}_t^\star} (\mathbf{\Delta}_t^\star) = 0 \notag \\
  & \mathbf{B}_t + \frac{\rho}{2} \norm{\mathbf{\Delta}_t^\star} I \succeq 0 . \notag
\end{align}
Together these conditions also imply that:
\begin{align}
  & (\mathbf{\Delta}_t^\star)^\top \mathbf{g}_t + (\mathbf{\Delta}_t^\star)^\top \mathbf{B}_t (\mathbf{\Delta}_t^\star) + \frac{\rho}{2} \norm{\mathbf{\Delta}_t^\star}^3 = 0 \notag \\
  & (\mathbf{\Delta}_t^\star)^\top \mathbf{B}_t (\mathbf{\Delta}_t^\star) + \frac{\rho}{2} \norm{\mathbf{\Delta}_t^\star}^3 \geq 0. \notag
\end{align}

Now immediately from the definition of stochastic cubic submodel model and the aforementioned conditions we have that:
      \begin{align}
        f(\mathbf{x}_t) - m_{t}({\mathbf{x}_t+\mathbf{\Delta}_t^\star}) & = -(\mathbf{\Delta}_t^\star)^\top \mathbf{g}_t - \frac{1}{2} (\mathbf{\Delta}_t^\star)^\top \mathbf{B}_t (\mathbf{\Delta}_t^\star) - \frac{\rho}{6} \norm{\mathbf{x}_t+\mathbf{\Delta}_t^\star}^3 \notag \\
                              & = \frac{1}{2} (\mathbf{\Delta}_t^\star)^\top \mathbf{B}_t (\mathbf{\Delta}_t^\star) + \frac{1}{3} \rho \norm{\mathbf{\Delta}_t^\star}^3 \notag \\
                              & \geq \frac{1}{12} \rho \norm{\mathbf{\Delta}_t^\star}^3 \notag
      \end{align}
      An identical statement appears as Lemma 10 in \cite{nesterov2006cubic}, so this is merely restated here for completeness.
\end{proof}

Thus to guarantee sufficient descent it suffices to lower bound the $\norm{\mathbf{\Delta}_t^\star}$. We now prove Lemma \ref{lem:s_big}, which guarantees the sufficient ``movement" for the exact update:
$\norm{\mathbf{\Delta}_t^\star}$. In particular this will allow us to show that when $\mathbf{x}_t + \mathbf{\Delta}_t^\star$ is not an $\epsilon$-second-order stationary point then $\norm{\mathbf{\Delta}_t^\star} \geq \frac{1}{2} \sqrt{\frac{\epsilon}{\rho}}$.
\sbig*
\begin{proof}
As a consequence of the global optimality condition, given in Equation~\eqref{eq:cubic_sol_stat}, we have that:
  \begin{align}
    \norm{\mathbf{g}_t + \mathbf{B}_t (\mathbf{\Delta}_t^\star)}  = \frac{\rho}{2} \norm{\mathbf{\Delta}_t^\star}^2. \label{eq:proof1}
  \end{align}
Moreover, from the Hessian-Lipschitz condition it follows that:
\begin{align}
  \norm{\nabla{f}(\mathbf{x}_t+\mathbf{\Delta}_t^\star)-\nabla f(\mathbf{x}_t) - \nabla^2 f(\mathbf{x}_t)(\mathbf{\Delta}_t^\star)} \leq \frac{\rho}{2} \norm{\mathbf{\Delta}_t^\star}^2. \label{eq:proof2}
\end{align}
Combining the concentration assumptions with Equation~\eqref{eq:proof1} and Inequality~\eqref{eq:proof2}, we obtain:
\begin{align}
  \norm{\nabla f(\mathbf{x}_t+\mathbf{\Delta}_t^\star)} & = \norm{\nabla f(\mathbf{x}_t+\mathbf{\Delta}_t^\star) - \nabla f(\mathbf{x}_t) - \nabla^2 f(\mathbf{x}_t)(\mathbf{\Delta}_t^\star)} + \norm{\nabla f(\mathbf{x}_t) + \nabla^2 f(\mathbf{x}_t)(\mathbf{\Delta}_t^\star)} \notag \\
                       & \leq \norm{\nabla f(\mathbf{x}_t+\mathbf{\Delta}_t^\star) - \nabla f(\mathbf{x}_t) - \nabla^2 f(\mathbf{x}_t)(\mathbf{\Delta}_t^\star)} + \norm{\mathbf{g}_t + \mathbf{B}_t (\mathbf{\Delta}_t^\star)} \notag \\
                       & + \norm{\mathbf{g}_t-\nabla f(\mathbf{x}_t)} + \norm{(\mathbf{B}_t-\nabla^2 f(\mathbf{x}_t))\mathbf{\Delta}_t^\star} \notag \\
                       & \leq \rho \norm{\mathbf{\Delta}_t^\star}^2 + c_1 \epsilon + c_2 \sqrt{\rho \epsilon} \norm{\mathbf{\Delta}_t^\star}. \label{eq:proof3}
\end{align}
An application of the Fenchel-Young inequality to the final term in Equation~\eqref{eq:proof3} then yields:
\begin{align}
  & \norm{\nabla f(\mathbf{x}_t+\mathbf{\Delta}_t^\star)} \leq \rho(1+\frac{c_2}{2}) \norm{\mathbf{\Delta}_t^\star}^2 + (c_1+\frac{c_2}{2}) \epsilon \implies \notag \\
  & \frac{1}{\rho(1+\frac{c_2}{2})} \left( \norm{\nabla f(\mathbf{x}_t+\mathbf{\Delta}_t^\star)} - (c_1+\frac{c_2}{2}) \epsilon \right) \leq \norm{\mathbf{\Delta}_t^\star}^2, \notag
\end{align}
which lower bounds $\norm{\mathbf{\Delta}_t^\star}$ with respect to the gradient at $\mathbf{x}_t$. For the corresponding Hessian lower bound we first utilize the Hessian Lipschitz condition:
\begin{align}
  \nabla^2 f(\mathbf{x}_t+\mathbf{\Delta}_t^\star) & \succeq \nabla^2 f(\mathbf{x}_t) - \rho \norm{\mathbf{\Delta}_t^\star} I \notag \\
                    & \succeq \mathbf{B}_t - c_2 \sqrt{\rho \epsilon} I - \rho \norm{\mathbf{\Delta}_t^\star} I \notag \\
                    & \succeq -c_2 \sqrt{\rho \epsilon} I - \frac{3}{2} \rho \norm{\mathbf{\Delta}_t^\star} I \notag,
\end{align}
followed by the concentration condition and the optimality condition \eqref{eq:cubic_sol_dual}. This immediately implies
\begin{align}
  & \norm{\mathbf{\Delta}_t^\star} I \succeq -\frac{2}{3 \rho} \left( \nabla^2 f(\mathbf{x}_t+\mathbf{\Delta}_t^\star) + c_2 \sqrt{\rho \epsilon} I \right) \implies \notag \\
  & \norm{\mathbf{\Delta}_t^\star} \geq -\frac{2}{3 \rho} \lambda_{\min}(\nabla^2 f(\mathbf{x}_t+\mathbf{\Delta}_t^\star)) - \frac{2c_2}{3 \sqrt{\rho}} \sqrt{\epsilon} \notag
\end{align}
Combining we obtain that:
\begin{align}
  \norm{\mathbf{\Delta}_t^\star} \geq \max \left \{ \sqrt{\frac{1}{\rho(1+\frac{c_2}{2})} \left( \norm{\nabla f(\mathbf{x}_t+\mathbf{\Delta}_t^\star)} - (c_1+\frac{c_2}{2}) \epsilon \right)}, -\frac{2}{3 \rho} \lambda_n(\nabla^2 f(\mathbf{x}_t+\mathbf{\Delta}_t^\star)) - \frac{2c_2 }{3 \sqrt{\rho}} \sqrt{\epsilon} \right \}. \notag
\end{align}

We consider the case of large gradient and large Hessian in turn (one of which must hold since $\mathbf{x}_t+\mathbf{\Delta}_t^\star$ is not an $\epsilon$-second-order stationary point). There exist $c_1, c_2$  in the following so that we can obtain:
\begin{itemize}
  \item If $\norm{\nabla f(\mathbf{x}_t+\mathbf{\Delta}_t^\star)} > \epsilon$, then we have that
  \begin{align}
      \norm{\mathbf{\Delta}_t^\star} > \sqrt{\frac{1}{\rho(1+\frac{c_2}{2})} \left(\norm{\nabla f(\mathbf{x}_t+\mathbf{\Delta}_t^\star)} - (c_1+\frac{c_2}{2})\epsilon \right)} \geq \sqrt{\frac{1-c_1-\frac{c_2}{2}}{1+\frac{c_2}{2}}} \sqrt{\frac{\epsilon}{\rho}} > \frac{1}{2} \sqrt{\frac{\epsilon}{\rho}}.
  \end{align}
  \item If $-\lambda_n(\nabla^2 f(\mathbf{x}_t+\mathbf{\Delta}_t^\star)) > \sqrt{\rho \epsilon}$, then we have that $\norm{\mathbf{\Delta}_t^\star} > \frac{2}{3} \sqrt{\frac{\epsilon}{\rho}} - \frac{2c_2}{3} \sqrt{\frac{\epsilon}{\rho}} = \frac{2}{3}(1-c_2) \sqrt{\frac{\epsilon}{\rho}} > \frac{1}{2} \sqrt{\frac{\epsilon}{\rho}}.$
\end{itemize}
We can similarly check the lower bounds directly stated are true. Choosing $c_1=\frac{1}{200}$ and $c_2=\frac{1}{200}$ will verify these inequalities for example.
\end{proof}

\subsection{Proof of Claim 1}
Here we provide a proof of statement equivalent to Claim 1 in the full, non-stochastic setting with approximate model minimization.
We focus on the case when the Cubic-Subsolver routine executes \textbf{Case 2}, since the result is vacuously true when the routine executes \textbf{Case 1}.
Our first lemma will both help show sufficient descent and provide a stopping condition for Algorithm \ref{algo:SCRN}.
For context, recall that when $\mathbf{x}_t + \mathbf{\Delta}_t^\star$ is not an $\epsilon$-second-order stationary point then $\norm{\mathbf{\Delta}_t^\star} \geq \frac{1}{2}\sqrt{\frac{\epsilon}{\rho}}$ by
Lemma \ref{lem:s_big}.

\begin{lemma} \label{lem:abstractcubicdescent}
        If the routine Cubic-Subsolver uses \textbf{Case 2}, and if $\norm{\mathbf{\Delta}_t^\star} \geq \frac{1}{2} \sqrt{\frac{\epsilon}{\rho}}$, then it will return a point $\mathbf{\Delta}$
        satisfying $m_{t}(\mathbf{x}_t+\mathbf{\Delta}_t) \leq m_{t}(\mathbf{x}_t) - \frac{1-c_3}{12} \rho \norm{\mathbf{\Delta}_t^\star}^3 \leq \frac{1-c_3}{96} \sqrt{\frac{\epsilon^3}{\rho}}$.
\end{lemma}
\begin{proof}
  In the case when $\norm{\mathbf{\Delta}_t^\star} \geq \frac{1}{2} \sqrt{\frac{\epsilon}{\rho}}$, by the definition of the routine Cubic-Subsolver
  we can ensure that $m_{t} (\mathbf{x}_t+\mathbf{\Delta}_t) \leq m_{t}(\mathbf{x}_t+\mathbf{\Delta}_t^\star) + \frac{c_3}{12} \rho \norm{\mathbf{\Delta}_t^\star}^3$ for arbitarily small $c_3$ using $\mathcal{T}(\epsilon)$ iterations.
  We can now combine the aforementioned display with Lemma \ref{lem:optdescent} (recalling that $m_{t}(\mathbf{x}_t) = f(\mathbf{x}_t)$) to conclude that:
      \begin{align}
          & m_{t}(\mathbf{x}_t+\mathbf{\Delta}_t) \leq m_{t}(\mathbf{x}_t+\mathbf{\Delta}_t^\star) + \frac{c_3}{12} \rho \norm{\mathbf{\Delta}_t^\star}^3  \notag \\
          & m_t(\mathbf{x}_t+\mathbf{\Delta}_t^\star) \leq m_{t}(\mathbf{x}_t)  - \frac{\rho}{12} \norm{\mathbf{\Delta}_t^\star}^3 \implies \\
          & m_t(\mathbf{x}_t+\mathbf{\Delta}_t)  \leq m_{t}(\mathbf{x}_t)  - (\frac{1-c_3}{12} ) \rho \norm{\mathbf{\Delta}_t^\star}^3 \leq m_{t}(\mathbf{x}_t) - \frac{(1-c_3)}{96} \sqrt{\frac{\epsilon^3}{\rho}}.
      \end{align}
      for suitable choice of $c_3$ which can be made arbitarily small.
      \end{proof}
\begin{claim} \label{claim:1}
  Assume we are in the setting of Lemma \ref{lem:conc_conds} with sufficiently small constants $c_1, c_2$. If $\mathbf{\Delta}$ is the output of the routine Cubic-Subsolver when executing \textbf{Case 2}
  and if $\mathbf{x}_t + \mathbf{\Delta}_t^\star$ is not an $\epsilon$-second-order stationary point of $f$, then $m_t(\mathbf{x}_t + \mathbf{\Delta}_t) - m_t(\mathbf{x}_t) \leq -\frac{1-c_3}{96} \sqrt{\frac{\epsilon^3}{\rho}}$.
\end{claim}
\begin{proof}
  This is an immediate consequence of Lemmas \ref{lem:s_big} and \ref{lem:abstractcubicdescent}.
\end{proof}

If we do not observe sufficient descent in the cubic submodel (which is not possible in \textbf{Case 1} by definition) then as a consequence of Claim \ref{claim:1} and Lemma \ref{lem:abstractcubicdescent}
we can conclude that $\norm{\mathbf{\Delta}_t^\star} \leq \frac{1}{2} \sqrt{\frac{\epsilon}{\rho}}$ and
that $\mathbf{x}_t + \mathbf{\Delta}_t^\star$ \textit{is} an
$\epsilon$-second-order stationary point. However, we cannot compute $\mathbf{\Delta}_t^\star$ directly. So instead we use a final gradient descent loop in Algorithm \ref{algo:SCD-f}, to ensure the final point returned in this scenario will be an
$\epsilon$-second-order stationary point up to a rescaling.
\begin{lemma}
        \label{lem:final_grad_loop}
        Assume we are in the setting of Lemma \ref{lem:conc_conds} with sufficiently small constants $c_1, c_2$.
        If $\mathbf{x}_t+\mathbf{\Delta}_t^\star$ is an $\epsilon$-second-order stationary point of $f$, and $\norm{\mathbf{\Delta}_t^\star} \leq \frac{1}{2} \sqrt{\frac{\epsilon}{\rho}}$, then
        Algorithm \ref{algo:SCD-f} will output a point $\mathbf{\Delta}$ such that $\mathbf{x}_{t+1} = \mathbf{x}_{t}+\mathbf{\Delta}$ is a $4 \epsilon$-second-order stationary point of $f$.
      \end{lemma}
      \begin{proof}
        Since $\mathbf{x}_t + \mathbf{\Delta}_t^\star$ is an $\epsilon$-second order stationary point of $f$, by gradient smoothness and the concentration conditions we have that
        $\norm{\mathbf{g}_t} \leq \norm{\nabla f(\mathbf{x}_t+\mathbf{\Delta}_t^\star)} + \ell \norm{\mathbf{\Delta}_t^\star} + \norm{\mathbf{g}_t - \nabla f(\mathbf{x}_t)} \leq (1+c_1) \epsilon + \frac{1}{2} \sqrt{\frac{\epsilon}{\rho}} \ell \leq (\frac{3}{2} + 1 + c_1) \frac{\ell^2}{\rho} \leq \frac{19}{16} \frac{\ell^2}{\rho}$
        for sufficiently small $c_1$. Then we can verify the step-size choice $\eta = \frac{1}{20} \ell$ and initialization at $\mathbf{\Delta}=0$ (in the centered coordinates) for the routine Cubic-FinalSubsolver verifies Assumptions A and B\footnote{See Appendix Section \ref{sec:carmon_subsolver} for more details.} in \citet{carmon2016gradient}.
        So, by Corollary 2.5 in \citet{carmon2016gradient}---which states the norms of the gradient descent iterates, $\norm{\mathbf{\Delta}}$, are non-decreasing and satisfy $\norm{\mathbf{\Delta}} \leq \norm{\mathbf{\Delta}_t^\star}$---we have that $\norm{\mathbf{\Delta}-\mathbf{\Delta}_t^\star} \leq 2 \norm{\mathbf{\Delta}_t^\star} \leq \sqrt{\frac{\epsilon}{\rho}}$.

        We first show that $- \lambda_{\min}(\nabla^2 f(\mathbf{x}_{t+1})) \lesssim \sqrt{\rho \epsilon}$.
        Since $f$ is $\rho$-Hessian-Lipschitz we have that:
        \begin{align}
          \nabla^2 f(\mathbf{x}_{t+1}) \succeq \nabla^2 f(\mathbf{x}_t+\mathbf{\Delta}_t^\star) - \rho 2 \norm{\mathbf{\Delta}_{t}^\star} I \succeq -2 \sqrt{\rho \epsilon} I. \notag
        \end{align}

        We now show that $\norm{\nabla f(\mathbf{x}_{t+1})} \lesssim \epsilon$ and thus also small.
        Once again using that $f$ is $\rho$-Hessian-Lipschitz (Lemma 1 in \cite{nesterov2006cubic}) we have that:
        \begin{align}
          & \norm{\nabla f(\mathbf{x}_{t+1}) - \nabla f(\mathbf{x}_t) - \nabla^2 f(\mathbf{x}_t) \mathbf{\Delta}} \leq \frac{\rho}{2} \norm{\mathbf{\Delta}}^2 \leq \frac{\rho}{2} \norm{\mathbf{\Delta}_t^\star}^2 \leq \frac{\epsilon}{8}. \notag
        \end{align}
        Now,  by the termination condition in Algorithm \ref{algo:SCD-f} we have that
        $\norm{\mathbf{g} + \mathbf{B} \mathbf{\Delta} + \frac{\rho}{2} \norm{\mathbf{\Delta}} \mathbf{\Delta}} < \frac{\epsilon}{2}$. So,
        \begin{align}
          \norm{\mathbf{g} + \mathbf{B} \mathbf{\Delta}} < \frac{\epsilon}{2} + \frac{\rho}{2} \norm{\mathbf{\Delta}}^2 \leq \frac{5}{8} \epsilon. \notag
        \end{align}
        Using gradient/Hessian concentration with the previous displays we also obtain that:
        \begin{align}
          & \norm{\nabla f(\mathbf{x}_{t+1})} - \norm{\mathbf{g}-\nabla f(\mathbf{x}_t)} -\norm{(\mathbf{B}-\nabla^2 f(\mathbf{x}_t)) \mathbf{\Delta}} - \norm{\mathbf{g}+\mathbf{B} \mathbf{\Delta}} \leq \norm{\nabla f(\mathbf{x}_{t+1}) - \nabla f(\mathbf{x}_t) - \nabla^2 f(\mathbf{x}_t) \mathbf{\Delta}} \notag \\
          & \implies \norm{\nabla f(\mathbf{x}_{t+1})} \leq \left(c_1 + \frac{c_2}{2} + \frac{5}{8} + \frac{1}{8} \right) \epsilon \leq \epsilon, \notag
        \end{align}
        for sufficiently small $c_1$ and $c_2$.

        Let us now bound the iteration complexity of this step. From our previous argument we have that $\norm{\mathbf{g}_t} \leq (1+c_1) \epsilon + \frac{\ell}{2 \sqrt{\rho}} \sqrt{\epsilon}$.
        Similarly, the concentration conditions imply $\norm{\mathbf{B}_t \mathbf{\Delta}_t^\star} \leq (\ell +  c_2 \sqrt{\rho \epsilon}) \norm{\mathbf{\Delta}_t^\star}$. Thus we have that
        $m_t(\mathbf{x}_t)-m_{t}(\mathbf{x}_t+ \mathbf{\Delta}_t^\star) = ((1+c_1) \epsilon + \frac{\ell}{2 \sqrt{\rho}} \sqrt{\epsilon}) \norm{\mathbf{\Delta}_t^\star} + \frac{1}{2}(\ell + c_2 \sqrt{\rho \epsilon})\norm{\mathbf{\Delta}_t^\star}^2 + \frac{\rho}{6} \norm{\mathbf{\Delta}_t^\star}^3 \leq \frac{3 \ell}{\rho} \epsilon + \left(\frac{1+c_1+4c_2}{8} + \frac{1}{48}\right) \sqrt{\frac{\epsilon^3}{\rho}} \leq \mathcal{O}(1) \cdot \frac{\epsilon \ell}{\rho}$
        since $c_1, c_2$ are numerical constants that can be made arbitrarily small.

        So by the standard analysis of gradient descent for smooth functions, see \citet{nesterov2013introductory} for example, we have that Algorithm \ref{algo:SCD-f} will terminate
        in at most $\leq \lceil \frac{m_t(\mathbf{x}_t) - m_{t}(\mathbf{x}_t+\mathbf{\Delta}_t^\star)}{\eta (\epsilon/2)^2} \rceil \leq \mathcal{O}(1) \cdot ( \frac{\ell^2}{ \rho \epsilon} )$ iterations.
        This will take at most $\mathcal{\tilde{O}}(\max(\frac{M_1}{\sqrt{\rho\epsilon}}, \frac{\sigma_2^2}{\epsilon}) \cdot \frac{\ell^2}{\rho \epsilon})$ Hessian-vector products and $ \mathcal{\tilde{O}}(\max( \frac{M_1}{\epsilon}, \frac{\sigma_1^2}{\epsilon^2}))$ gradient evaluations
        which will be subleading in the overall complexity.
      \end{proof}

\subsection{Proof of Claim 2}
We now prove our main descent lemma equivalent to \textbf{Claim 2}---this will show if the cubic submodel has a large decrease, then the underlying true function must also have large decrease.
As before we focus on the case when the Cubic-Subsolver routine
executes \textbf{Case 2} since the result is vacuously true in \textbf{Case 1}.
      \begin{claim}
        \label{claim:2}
        Assume we are in the setting of Lemma \ref{lem:conc_conds} with sufficiently small constants $c_1, c_2$. If the Cubic-Subsolver routine uses \textbf{Case 2}, and if
        $m_{t}(\mathbf{x}_t + \mathbf{\Delta}_t) - m_t(\mathbf{x}_t) \leq -(\frac{1-c_3}{96}) \sqrt{\frac{\epsilon^3}{\rho}}$, then
        $f(\mathbf{x}_{t} + \mathbf{\Delta}_t)- f(\mathbf{x}_t) \leq -\left( \frac{1 - c_3 - c_5}{96} \right) \sqrt{\frac{\epsilon^3}{\rho}}$.
      \end{claim}
      \begin{proof}
        Using that $f$ is $\rho$-Hessian Lipschitz (and hence admits a cubic majorizer by Lemma 1 in \citet{nesterov2006cubic} for example) as well as the concentration conditions we have that:
        \begin{align}
          & f(\mathbf{x}_{t}+\mathbf{\Delta}_t) \leq f(\mathbf{x}_{t}) + \nabla f(\mathbf{x}_{t})^\top \mathbf{\Delta}_{t} + \frac{1}{2} \mathbf{\Delta}_{t}^\top \nabla^2 f(\mathbf{x}_{t}) \mathbf{\Delta}_{t} + \frac{\rho}{6} \norm{\mathbf{\Delta}_{t}}_2^3 \implies \notag \\
          & f(\mathbf{x}_{t}+\mathbf{\Delta}_t) - f(\mathbf{x}_t) \leq m_{t}(\mathbf{x}_t+\mathbf{\Delta}_t) - m_{t}(\mathbf{x}_t) + (\nabla f(\mathbf{x}_t) - \mathbf{g}_t)^\top \mathbf{\Delta}_{t} + \frac{1}{2} \mathbf{\Delta}_{t}^\top(\mathbf{B}_{t}-\nabla^2 f(\mathbf{x}_{t})) \mathbf{\Delta}_{t} \notag \\
          & \leq m_{t}(\mathbf{x}_t+\mathbf{\Delta}_t) - m_{t}(\mathbf{x}_t) + c_1 \epsilon \norm{\mathbf{\Delta}_{t}} + \frac{c_2}{2} \sqrt{\rho \epsilon} \norm{\mathbf{\Delta}_{t}}^2  \notag \\
          & \leq m_{t}(\mathbf{x}_t+\mathbf{\Delta}_t) - m_{t}(\mathbf{x}_t) + c_1 \epsilon \left(\norm{\mathbf{\Delta}_{t}^\star} + c_4 \sqrt{\frac{\epsilon}{\rho}} \right) + \frac{c_2}{2} \sqrt{\rho \epsilon} \left( \norm{\mathbf{\Delta}_{t}^\star}+c_4 \sqrt{\frac{\epsilon}{\rho}} \right)^2 \notag \\
          & \leq m_{t}(\mathbf{x}_t+\mathbf{\Delta}_t) - m_{t}(\mathbf{x}_t) + (c_1 + c_2 c_4) \epsilon \norm{\mathbf{\Delta}_{t}^\star} + \frac{c_2 c_4^2}{2} \sqrt{\rho \epsilon} \norm{\mathbf{\Delta}_t^\star}^2+(c_1+\frac{c_2 c_4}{2}) c_4 \sqrt{\frac{\epsilon^3}{\rho}}, \label{eq:general_descent}
        \end{align}
        since by the definition the Cubic-Subsolver routine, when we use \textbf{Case 2} we have that $\norm{\mathbf{\Delta}_t} \leq \norm{\mathbf{\Delta}_t^\star} + c_4 \sqrt{\frac{\epsilon}{\rho}}$.
        We now consider two different situations -- when $\norm{\mathbf{\Delta}_t^\star} \geq \frac{1}{2} \sqrt{\frac{\epsilon}{\rho}}$ and when $\norm{\mathbf{\Delta}_t^\star} \leq \frac{1}{2} \sqrt{\frac{\epsilon}{\rho}}$.

        First, if $\norm{\mathbf{\Delta}_t^\star} \geq \frac{1}{2} \sqrt{\frac{\epsilon}{\rho}}$ then by Lemma \ref{lem:abstractcubicdescent} we may assume the stronger guarantee that $m_{t}(\mathbf{x}_t + \mathbf{\Delta}_t) - m_t(\mathbf{x}_t) \leq -(\frac{1-c_3}{12}) \rho \norm{\mathbf{\Delta}_t^\star}^3$.
        So by considering the above display in Equation \eqref{eq:general_descent} we can conclude that:
        \begin{align}
          & f(\mathbf{x}_{t}+\mathbf{\Delta}_t) - f(\mathbf{x}_t) \leq m_{t}(\mathbf{x}_t+\mathbf{\Delta}_t) - m_{t}(\mathbf{x}_t) + (c_1 + c_2 c_4) \epsilon \norm{\mathbf{\Delta}_{t}^\star} + \frac{c_2 c_4^2}{2} \sqrt{\rho \epsilon} \norm{\mathbf{\Delta}_t^\star}^2+(c_1+\frac{c_2 c_4}{2}) c_4 \sqrt{\frac{\epsilon^3}{\rho}} \notag \\
          & \leq -\left( \frac{1 - c_3 - 48 (c_1+c_2 c_4) - 12 c_2 c_4^2}{12} \right) \rho \norm{\mathbf{\Delta}_t^\star}^3 + \left(c_1+\frac{c_2 c_4}{2} \right) c_4 \sqrt{\frac{\epsilon^3}{\rho}} \notag \\
          & \leq -\left( \frac{1 - c_3 - 48 c_1  - 48 c_2 c_4 - 96 c_1 c_4 - 60 c_2 c_4^2}{96} \right) \sqrt{\frac{\epsilon^3}{\rho}}, \notag
        \end{align}
        since the numerical constants $c_1, c_2, c_3$ can be made arbitrarily small.

        Now, if $\norm{\mathbf{\Delta}_t^\star} \leq \frac{1}{2} \sqrt{\frac{\epsilon}{\rho}}$, we directly use the assumption that $m_{t}(\mathbf{x}_t + \mathbf{\Delta}_t) - m_t(\mathbf{x}_t) \leq -(\frac{1-c_3}{96}) \sqrt{\frac{\epsilon^3}{\rho}}$. Combining with the display in
        in Equation \eqref{eq:general_descent} we can conclude that:
        \begin{align}
          & f(\mathbf{x}_{t}+\mathbf{\Delta}_t) - f(\mathbf{x}_t) \leq m_{t}(\mathbf{x}_t+\mathbf{\Delta}_t) - m_{t}(\mathbf{x}_t) + (c_1 + c_2 c_4) \epsilon \norm{\mathbf{\Delta}_{t}^\star} + \frac{c_2 c_4^2}{2} \sqrt{\rho \epsilon} \norm{\mathbf{\Delta}_t^\star}^2+(c_1+\frac{c_2 c_4}{2}) c_4 \sqrt{\frac{\epsilon^3}{\rho}} \notag \\
          & \leq - \left(\frac{1-c_3}{96}\sqrt{\frac{\epsilon^3}{\rho}} \right) + \left ( (c_1 + c_2 c_4) \epsilon \cdot \frac{1}{2} \sqrt{\frac{\epsilon}{\rho}} + \frac{c_2 c_4^2}{2} \sqrt{\rho \epsilon} \cdot \frac{1}{4}\frac{\epsilon}{\rho}+(c_1+\frac{c_2 c_4}{2}) c_4 \sqrt{\frac{\epsilon^3}{\rho}} \right) \notag \\
          & \leq -\left( \frac{1 - c_3 - 48 c_1  - 48 c_2 c_4 - 96 c_1 c_4 - 60 c_2 c_4^2}{96} \right) \sqrt{\frac{\epsilon^3}{\rho}}, \notag
        \end{align}
        since the numerical constants $c_1, c_2, c_3$ can be made arbitrarily small. Indeed, recall that $c_1$ is the gradient concentration constant, $c_2$ is the Hessian-vector product concentration constant, and $c_3$
        is the tolerance of the Cubic-Subsolver routine when using \textbf{Case 2}.
        Thus, in both situations, we have that:
        \begin{align}
          f(\mathbf{x}_{t}+\mathbf{\Delta}_t) - f(\mathbf{x}_t) \leq \frac{1-c_3-c_5}{96} \sqrt{\frac{\epsilon^3}{\rho}},
        \end{align}
        denoting $c_5 = 48 c_1  - 48 c_2 c_4 - 96 c_1 c_4 - 60 c_2 c_4^2$ for notational convenience (which can also be made arbitrarily small for sufficiently small $c_1, c_2$).
      \end{proof}

\subsection{Proof of Theorem \ref{thm:main1}}

We now prove the correctness of Algorithm \ref{algo:SCRN}. We assume, as usual, the underlying function $f(x)$ possesses a lower bound $f^*$.

    \thmmainabstract*
      \begin{proof}
        For notational convenience let \textbf{Case 1} of the routine Cubic Subsolver satisfy:
        \begin{align}
          \max \{ f(\mathbf{x}_t + \mathbf{\Delta}_t) - f(\mathbf{x}_t), m_t(\mathbf{x}_t+\mathbf{\Delta}_t) - m_t(\mathbf{x}_t) \} \leq -K_1 \sqrt{\frac{\epsilon^3}{\rho}}. \notag
        \end{align}
        and use $K_2 = \frac{1-c_3}{96}$ to denote the descent constant of the cubic submodel in the assumption of Claim \ref{claim:2}. Further, let $K_{\text{prog}} = \min \{\frac{1-c_3-c_5}{96}, K_1 \}$ which
        we will use as the progress constant corresponding to descent in the underlying function $f$.
        Without loss of generality, we assume that $-K_1 \leq -K_2$ for convenience in the proof. If $-K_1 \geq -K_2$, we can simply rescale the descent constant corresponding to \textbf{Case 2} for the cubic submodel, $\frac{1-c_3}{96}$,
        to be equal to $-K_1$, which will require shrinking $c_1, c_2$ proportionally to ensure that the rescaled version of the function descent constant, $\frac{1-c_3-c_5}{96}$, is positive.

        Now, we choose $c_1, c_2, c_3$ so that $K_{2} > 0$, $K_{\text{prog}} > 0$, and Lemma \ref{lem:s_big} holds in the aforementioned form.
        For the correctness of Algorithm \ref{algo:SCRN} we choose the numerical constant in Line \ref{algo1:line:small_descent} as $K_2$ -- so the ``if statement'' checks the condition $\Delta m = m_{t}(\mathbf{x}_{t+1}) - m_{t}(\mathbf{x}_t) \geq -K_{2} \sqrt{\frac{\epsilon'^3}{\rho}}$. Here we use
        a rescaled $\epsilon'=\frac{1}{4}\epsilon$ for the duration of the proof.

        At each iteration the event that the setting of Lemma \ref{lem:conc_conds} hold has probability greater then $1-2 \delta'$. Conditioned on this event
        let the routine Cubic-Subsolver have a further probability of at most $\delta'$ of failure.
        We now proceed with our analysis deterministically conditioned on the event $E$ -- that at each iteration the concentration conditions hold and the routine Cubic-Subsolver succeeds -- which has probability
        greater then $1-3 \delta' T_{\text{outer}} \geq 1-\delta$ by a union bound for $\delta' = \frac{\delta}{3 T_{\text{outer}}}$.

        Let us now bound the iteration complexity of Algorithm \ref{algo:SCRN} as $T_{\text{outer}}$. We cannot have the ``if statement'' in Line \ref{algo1:line:small_descent} fail indefinitely.
        At a given iteration, if the routine Cubic-Subsolver outputs a point $\mathbf{\Delta}$ that satisfies
        \begin{align}
          m_{t}(\mathbf{x}_t+\mathbf{\Delta}_t) - m_{t}(\mathbf{x}_t) \leq - K_{2} \sqrt{\frac{\epsilon'^3}{\rho}} \notag
        \end{align}
        then by Claim \ref{claim:2} and the definition of \textbf{Case 1} of the Cubic-Subsolver we also have that:
        \begin{align}
          f(\mathbf{x}_t+\mathbf{\Delta}_t) - f(\mathbf{x}_t) \leq - K_{\text{prog}} \sqrt{\frac{\epsilon'^3}{\rho}}. \notag
        \end{align}
        Note  if the Cubic-Subsolver uses \textbf{Case 1} in this iteration then we will vacuously achieve descent in both the underlying function $f$,
        and descent in the cubic submodel greater $-K_1 \sqrt{\frac{\epsilon'^3}{\rho}}$. Since $-K_1 \sqrt{\frac{\epsilon'^3}{\rho}} \leq -K_2 \sqrt{\frac{\epsilon'^3}{\rho}}$ by assumption, the algorithm will not terminate early
        at this iteration.
        Since the function $f$ is bounded below by $f^*$, the event $m_{t}(\mathbf{x}_t+\mathbf{\Delta}_t) - m_t(\mathbf{x}_t) \leq -K_{2} \sqrt{\frac{\epsilon'^3}{\rho}}$ which implies
        $f(\mathbf{x}_{t+1}) - f(\mathbf{x}_t) \leq -K_{\text{prog}} \sqrt{\frac{\epsilon'^3}{\rho}}$ can happen at most $T_{\text{outer}} = \lceil \frac{\sqrt{\rho}(f(x_0)-f^*)}{K_{\text{prog}}\epsilon'^{1.5}} \rceil$ times.

        Thus in the $T_{\text{outer}}$ iterations of Algorithm \ref{algo:SCRN} it must be the case that there is at least one iteration $T$, for which
        \begin{align}
        m_{T}(\mathbf{x}_{T}+\mathbf{\Delta}_{T}) - m_T(\mathbf{x}_{T}) \geq -K_{2} \sqrt{\frac{\epsilon'^3}{\rho}}. \notag
        \end{align}
        By the definition of the Cubic-Subsolver procedure and assumption that $-K_1 \leq -K_2$, it must be the case at iteration $T$ the routine Cubic-Subsolver used \textbf{Case 2}.
        Now by appealing to Claim \ref{claim:1} and Lemma \ref{lem:abstractcubicdescent} we must have that $\norm{\mathbf{\Delta}_T^\star} \leq \frac{1}{2} \sqrt{\frac{\epsilon'}{\rho}}$ and that $\mathbf{x}_{T} + \mathbf{\Delta}_{T}^\star$ is an $\epsilon'$-second-order stationary point of $f$.
        As we can see in Line \ref{algo1:line:small_descent} of Algorithm \ref{algo:SCRN}, at iteration $T$ the ``if statement'' will be true. Hence Algorithm \ref{algo:SCRN} will run the final gradient descent loop (Algorithm \ref{algo:SCD-f}) at iteration $T$, return the final point and proceed to exit via the break statement.
        Since the hypotheses of Lemma \ref{lem:final_grad_loop} are satisfied\footnote{We can also see with this rescaled $\epsilon'$ the step-size requirement in Lemma \ref{lem:final_grad_loop} will be satisfied.} at iteration $T$, Algorithm \ref{algo:SCD-f} will return a final point
        that is an $\epsilon$-second-order stationary point of $f$ as desired. We can verify the global constant $c = \min \{ \frac{K_{\text{prog}} }{8} , c_1, c_2 \}$ satisfies the conditions of the theorem.
      \begin{remark}
        We can also now do a careful count of the complexity of Algorithm \ref{algo:SCRN}.
        First, note at each outer iteration of Algorithm \ref{algo:SCRN} we require
        $n_1 \geq \max \left( \frac{M_1}{c_1 \epsilon}, \frac{\sigma_1^2}{c_1^2 \epsilon^2} \right) \frac{8}{3} \log \frac{2d}{\delta'}$ 
        samples to approximate the gradient and
        and $n_2 \geq \max( \frac{M_2}{c_2 \sqrt{\rho \epsilon}}, \frac{\sigma_2^2}{c_2^2 \rho \epsilon}) \frac{8}{3} \log \frac{2d}{\delta'}$
        to approximate the Hessian. The union bound stipulates we should take $\delta'(\epsilon) = \frac{\delta}{3 T_{\text{outer}}}$ to control the total failure probability of Algorithm \ref{algo:SCRN}.
        Then as we can see in the Proof of Theorem \ref{thm:main1}, Algorithm \ref{algo:SCRN} will terminate in
        at most
        \begin{align}
        T_{\text{outer}} = \lceil \frac{8 K_{\text{prog}} \sqrt{\rho}(f(x_0)-f^*)}{\epsilon^{3/2}} \rceil \label{eq:outer_iter_complexity}
        \end{align}
        iterations.
        The inner iteration complexity of the Cubic-Subsolver routine is $\mathcal{T}(\epsilon)$. The routine only requires computing the gradient vector once,
        but recomputes Hessian-vector products at each iteration.
        
        So the gradient complexity becomes
        \begin{align}
          & T_{\text{G}} \lesssim  \mathcal{T}(\epsilon) \times \frac{\sqrt{\rho}(f(x_0)-f^*)}{\epsilon^{1.5}} \times \max \left( \frac{M_1}{c_1 \epsilon}, \frac{\sigma_1^2}{c_1^2 \epsilon^2} \right) \frac{8}{3} \log \frac{2d}{\delta'} \notag \\
          & \sim \mathcal{\tilde{O}}\left(\frac{\sqrt{\rho} \sigma_1^2 (f(x_0)-f^*)}{\epsilon^{3.5}}\right) \text{ for } \epsilon \leq \frac{\sigma_1^2}{c_1 M_1}. \notag
        \end{align}
        Note that $\tilde{O}$ hides logarithmic factors since $\delta'(\epsilon) = \frac{\delta}{3 T_{\text{outer}}}$.

        The total complexity of Hessian-vector product evaluations is:
        \begin{align}
          & T_{\text{HV}} \lesssim \frac{K_{\text{prog}}\sqrt{\rho}(f(x_0)-f^*)}{\epsilon^{1.5}} \times \max( \frac{M_2}{c_2 \sqrt{\rho \epsilon}}, \frac{\sigma_2^2}{c_2^2 \rho \epsilon}) \frac{8}{3} \log \frac{2d}{\delta'} \notag \\
          & \sim \mathcal{\tilde{O}}\left( \mathcal{T}(\epsilon) \frac{\sigma_2^2 (f(x_0)-f^*)}{\sqrt{\rho} \epsilon^3}  \right) \text{ for } \epsilon \leq \frac{\sigma_2^4}{c_2^2 M_2^2 \rho}. \notag
        \end{align}
       
        Finally, recall the proof of Lemma \ref{lem:final_grad_loop} which shows total complexity of the final gradient descent loop, in Algorithm \ref{algo:SCD-f},
        will be subleading in overall gradient and Hessian-vector product complexity. 
        As before, we can verify the global constant $c = \min \{ \frac{K_{\text{prog}} }{8} , c_1, c_2 \}$ satisfies the conditions of the Theorem \ref{thm:main1}.
    \end{remark}
     \end{proof}
\section{Gradient Descent as a Cubic Subsolver}

Here we provide the proofs of Lemma \ref{lem:subsolveerr} and Corollary \ref{thm:main2}. In particular, we show Algorithm $\ref{algo:SCD}$ is a Cubic-Subsolver routine satisfying Condition \ref{cond:cubic_subsolver}.
We break the analysis into two cases showing that whenever $\norm{\mathbf{g}} \geq \frac{\ell^2}{\rho}$ taking a Cauchy step
satisfies the conditions of a \textbf{Case 1} procedure, and when $\norm{\mathbf{g}} \leq \frac{\ell^2}{\rho}$, running gradient descent on the cubic submodel satisfies the conditions of a \textbf{Case 2} procedure. Showing the latter relies on Theorem 3.2 from \citet{carmon2016gradient}.

\subsection{Cauchy Step in Algorithm \ref{algo:SCD}}
First, we argue that when the stochastic gradients are large -- $\norm{\mathbf{g}_t} \geq \frac{\ell^2}{\rho}$ -- the Cauchy step achieves sufficient descent in both $f$ and $\tilde{m}$. Thus the Cauchy step will satisfy Equation \ref{cond:oracle_append} of Condition \ref{cond:cubic_subsolver}, so it satisfies the conditions of
a \textbf{Case 1} procedure. As before, $\tilde{m}(\bDelta) = \bDelta^\top \mathbf{g} + \frac{1}{2}\bDelta^\top \mathbf{B}[\bDelta] + \frac{\rho}{6} \norm{\bDelta}^3$, refers to stochastic cubic submodel in centered coordinates.

First, recall several useful results: the \textit{Cauchy radius} is the magnitude of the global minimizer in the subspace spanned by $\mathbf{g}$ (which is analytically tractable for this cubic submodel).
Namely $R_c = \argmin_{\eta \in \mathbb{R}^d} \tilde{m} \left(-\eta \frac{\mathbf{g}}{\norm{\mathbf{g}}} \right)$ and a short computation shows that
\begin{align}
R_c = -\frac{\mathbf{g}^\top \mathbf{B} \mathbf{g}}{\rho \norm{\mathbf{g}}^2} + \sqrt{\left( \frac{\mathbf{g}^\top \mathbf{B} \mathbf{g}}{\rho \norm{\mathbf{g}}^2} \right)^2 + \frac{2 \norm{\mathbf{g}}}{\rho}} \notag
\end{align}
and that
\begin{align}
  \tilde{m} \left(-R_c \frac{\mathbf{g}}{\norm{\mathbf{g}}} \right) = -\frac{1}{2} \norm{\mathbf{g}} R_c - \frac{\rho}{12} R_c^3. \notag
\end{align}

\begin{lemma} \label{lem:cauchy}
    Assume we are in the setting of Lemma \ref{lem:conc_conds} with sufficiently small constants $c_1, c_2$. If $\norm{\mathbf{g}} \geq \frac{\ell^2}{\rho}$, the Cauchy step defined by taking $\mathbf{\Delta} = -R_c \frac{\mathbf{g}}{\norm{\mathbf{g}}}$
    will satisfy $\tilde{m} \left( \mathbf{\Delta} \right)  \leq -K_{1} \sqrt{\frac{\epsilon^3}{\rho}}$ and
    $f \left(\mathbf{x}_t + \mathbf{\Delta} \right) - f(\mathbf{x}_t) \leq -K_{1} \sqrt{\frac{\epsilon^3}{\rho}}$ for $K_{1} = -\frac{7}{20}$.  Thus, if Algorithm \ref{algo:SCD} takes a Cauchy step (when $\norm{\mathbf{g}} \geq \frac{\ell^2}{\rho}$)
    it will satisfy the conditions of an \textbf{Case 1} procedure.
\end{lemma}
\begin{proof}
    We first lower bound the Cauchy radius
    \begin{align}
    R_c = -\frac{\mathbf{g}^\top \mathbf{B} \mathbf{g}}{\rho \norm{\mathbf{g}}^2} + \sqrt{\left( \frac{\mathbf{g}^\top \mathbf{B} \mathbf{g}}{\rho \norm{\mathbf{g}}^2} \right)^2 + \frac{2 \norm{\mathbf{g}}}{\rho}} \geq \frac{1}{\rho} \left(-\frac{\mathbf{g}^\top \mathbf{B} \mathbf{g}}{\norm{\mathbf{g}}^2} + \sqrt{\left( \frac{\mathbf{g}^\top \mathbf{B} \mathbf{g}}{\norm{\mathbf{g}}^2} \right)^2 + 2 \ell^2 }\right). \notag
    \end{align}
    Note the function $-x+\sqrt{x^2+2}$ is decreasing over its support. Additionally, $\norm{\nabla^2 f(0)} \leq \ell$ and
    $\norm{(\mathbf{B}-\nabla^2 f(0)) \mathbf{g}} \leq c_2 \sqrt{\rho \epsilon} \norm{\mathbf{g}} \leq c_2 \ell \norm{\mathbf{g}}$ imply that $\frac{\mathbf{g}^\top B \mathbf{g}}{\norm{\mathbf{g}}^2} \leq (1+c_2) \ell$. So, combining with the previous display we obtain:
    \begin{align}
      R_c \geq \frac{\ell}{\rho}\left( -(1+c_2) + \sqrt{(1+c_2)^2+2} \right) \geq \frac{7}{10} \frac{\ell}{\rho}, \label{eq:cauchy_lower}
    \end{align}
    for sufficiently small $c_2$.
    So,
    \begin{align}
    \tilde{m} \left(-R_c \frac{\mathbf{g}}{\norm{\mathbf{g}}} \right) = -\frac{1}{2} \norm{\mathbf{g}} R_c - \frac{\rho}{12} R_c^3 \leq -\frac{7}{20} \frac{\ell^3}{\rho^2} \leq - \frac{7}{20} \sqrt{\frac{\epsilon^3}{\rho}}, \notag
    \end{align}
    for sufficiently small $c_2$.

    Now, at iteration $t$, using the $\rho$-Hessian Lipschitz condition and concentration conditions we obtain:
    \begin{align}
    & f \left( \mathbf{x}_{t}-R_c \frac{\mathbf{g}}{\norm{\mathbf{g}}} \right) - f(\mathbf{x}_t) \leq \tilde{m} \left(-R_c \frac{\mathbf{g}}{\norm{\mathbf{g}}} \right) + c_1 R_c \epsilon + \frac{1}{2} R_c^2 c_2 \sqrt{\rho \epsilon} \notag \\
    & \leq -\frac{1}{2} \norm{\mathbf{g}} R_c - \frac{\rho}{12} R_c^3 + c_1 R_c \epsilon + \frac{1}{2} R_c^2 c_2 \sqrt{\rho \epsilon} \notag \\
    & \leq -\frac{1}{2} \norm{\mathbf{g}} R_c - \frac{\ell^3}{\rho^2} \cdot \frac{\rho}{\ell} R_c \left(\frac{1}{12}(\frac{\rho}{\ell} R_c)^2-c_1-\frac{c_2}{2} (\frac{\rho}{\ell} R_c) \right) \notag.
    \end{align}
    For sufficiently small $c_1, c_2$ the quadratic $\frac{1}{12} x^2 - c_1 - \frac{c_2}{2}x$ is minimized at $x^*=\frac{7}{10}$ when restricted to $x \geq \frac{7}{10}$.
    Since $\frac{\rho}{\ell} R_c \geq \frac{7}{10}$ for sufficiently small $c_1, c_2$ as in Equation \eqref{eq:cauchy_lower}, we can also conclude that
    \begin{align}
      f \left( \mathbf{x}_{t}-R_c \frac{\mathbf{g}}{\norm{\mathbf{g}}} \right) - f(\mathbf{x}_t) \leq -\frac{7}{20} \frac{\ell^3}{\rho^2} \leq -\frac{7}{20} \sqrt{\frac{\epsilon^3}{\rho}}.
    \end{align}
    This establishes sufficient decrease with respect to the true function $f$.
    We can verify choosing $c_1=c_2=1/200$ satisfies all the inequalities in this section.
    Lastly, note the complexity of computing the Cauchy step is in fact $\mathcal{O}(1)$.
\end{proof}

\subsection{Gradient Descent Loop in Algorithm \ref{algo:SCD}} \label{sec:carmon_subsolver}
We now establish complexity of sufficient descent when we are using the gradient descent subsolver from \citet{carmon2016gradient} to minimize the stochastic cubic submodel in Algorithm \ref{algo:SCD} in the regime $\norm{\mathbf{g}} \leq \frac{\ell^2}{\rho}$.
\citet{carmon2016gradient} consider the gradient descent scheme\footnote{Note there is also a factor of $2$ difference used in the Hessian-Lipschitz constant in \citet{carmon2016gradient}. Namely the $2\rho'=\rho$ for every $\rho'$ that appears in \citet{carmon2016gradient}.}:
\begin{align}
  \mathbf{\Delta}_{t+1} = \mathbf{\Delta}_{t} - \eta \nabla \tilde{m}(\mathbf{\Delta}_{t}) \label{eq:gd}
\end{align}
for $\tilde{m}(\mathbf{\Delta}) = \mathbf{g}^\top \mathbf{\Delta} + \frac{1}{2} \mathbf{\Delta}^\top \mathbf{B} \mathbf{\Delta} + \frac{\rho}{6} \norm{\mathbf{\Delta}^3}$.
Here we let the index $t$ range over the iterates of the gradient descent loop for a fixed cubic submodel.
Defining $\beta = \norm{\mathbf{B}}$ and $R = \frac{\beta}{\rho} + \sqrt{(\frac{\beta}{\rho})^2+\frac{2 \norm{\mathbf{g}}}{\rho}}$
they make the following assumptions to show convergence:
  \begin{itemize}
    \item \textbf{Assumption A}: The step size for the gradient descent scheme satisfies
    \begin{align}
    0 <  \eta < \frac{1}{4(\beta + \frac{\rho}{2} R)}.
    \end{align}
    \item \textbf{Assumption B}: The initialization for the gradient descent scheme $\mathbf{\Delta}$, satisfies $\mathbf{\Delta} = -r \frac{\mathbf{g}}{\norm{\mathbf{g}}}$, with $0 \leq r \leq R_c$.
  \end{itemize}

  Then, Theorem 3.2 in \citet{carmon2016gradient} (restated here for convenience) gives that:
  \begin{theorem}[\citet{carmon2016gradient}]
    Let $q$ be uniformly distributed on the the unit sphere in $\mathbb{R}^d$. Then,  the iterates $\mathbf{\Delta}_{t}$ generated by the gradient descent scheme in Equation \eqref{eq:gd} satisfying Assumptions A and B
    obtained by replacing $\mathbf{g} \to \mathbf{g} + \sigma q$ for $\sigma = \frac{\rho \epsilon'}{\beta + \rho \norm{\mathbf{\Delta}^\star}} \cdot \frac{\bar{\sigma}}{12}$ with $\bar{\sigma} \leq 1$,
    satisfy $\tilde{m}(\mathbf{\Delta}_{t}) \leq \tilde{m}(\mathbf{\Delta}^\star) + (1+\bar{\sigma}) \epsilon'$ for all:
    \begin{align}
      & t \geq \mathcal{T}(\epsilon') = \frac{1+\bar{\sigma}}{\eta} \min \left \{ \frac{1}{\frac{\rho}{2} \norm{\mathbf{\Delta}^\star}-\gamma}, \frac{10 \norm{\mathbf{\Delta}^\star}^2}{\epsilon'} \right \} \times \notag \\
      & \left[ 6 \log \left( 1+\mathbb{I}_{\{ \gamma > 0 \}} \frac{3 \sqrt{d}}{\bar{\sigma} \delta'} \right) + 14 \log \left( \frac{(\beta+\rho \norm{\mathbf{\Delta}^\star}) \norm{\mathbf{\Delta}^\star}^2}{\epsilon'} \right)  \right], \label{eq:inner_iters}
    \end{align}
    with probability $1-\delta'$.
  \end{theorem}
  Note that the iterates $\mathbf{\Delta}_{t}$ are iterates generated from the solving the \textit{perturbed} cubic subproblem with $\mathbf{g} \to \mathbf{g} + \sigma q$---not from solving the original cubic subproblem. This step is necessary to avoid the ``hard case" of the non-convex quadratic problems. See \citet{carmon2016gradient} for more details.

    To apply this bound we will never have access to $\norm{\mathbf{\Delta}^\star}$ apriori. However, in the present
    we need only to use this lemma to conclude sufficient descent when $\mathbf{\Delta}^\star$ is not an $\epsilon$-second-order stationary point and hence when $\norm{\mathbf{\Delta}}^\star \geq \frac{1}{2} \sqrt{\frac{\epsilon}{\rho}}$.
    \\
    \begin{lemma}
      \label{lem:duchi_subsolver}
      Assume we are in the setting of Lemma \ref{lem:conc_conds} with sufficiently small constants $c_1, c_2$. Further, assume that $\norm{\mathbf{g}} \leq \frac{\ell^2}{\rho}$, and that we choose $\eta = \mathcal{O}(\frac{1}{\ell})$
      in the gradient descent iterate scheme in Equation \eqref{eq:gd} with perturbed gradient $\tilde{\mathbf{g}} = \mathbf{g} + \sigma \mathbf{q}$ for $\sigma = \mathcal{O}(\frac{\sqrt{\epsilon \rho}}{\ell})$, and $q \sim \text{Unif}(S^{d-1})$.
      If $\norm{\mathbf{\Delta}^\star} \geq \frac{1}{2} \sqrt{\frac{\epsilon}{\rho}}$, then the gradient descent iterate $\mathbf{\Delta}_{\mathcal{T}(\epsilon)}$ for $\mathcal{T}(\epsilon) = \mathcal{\tilde{O}}(\frac{\ell}{\sqrt{\rho \epsilon}})$
      will satisfy $\tilde{m}(\mathbf{\Delta}) \leq \tilde{m}(\mathbf{\Delta}^\star) + \frac{c_3}{12} \norm{\mathbf{\Delta}^\star}^3$. Further, the iterate $\mathbf{\Delta}_{\mathcal{T}(\epsilon)}$ will always satisfy
      $\norm{\mathbf{\Delta}_{\mathcal{T}(\epsilon)}} \leq \norm{\mathbf{\Delta}^\star} + c_4 \sqrt{\frac{\epsilon}{\rho}}$.
    \end{lemma}
    \begin{proof}
    To show the first statement we simply invoke Theorem 3.2 of \citet{carmon2016gradient} and check that its assumptions are satisfied. We assume that $\norm{\mathbf{\Delta}^\star} \geq \frac{1}{2} \sqrt{\frac{\epsilon}{\rho}} \equiv L$, $\norm{\mathbf{g}} \leq \frac{\ell^2}{\rho}$ in the present situation.
    We now set $\epsilon' = c_3 \frac{\rho}{24} \norm{\mathbf{\Delta}^\star}^3$ where we assume that $c_3$ and will be small. Then we take:
    \begin{align}
      \sigma &= c_3 \frac{\rho^2 L^3}{24 \cdot 12 (2 \ell+\rho L)} = \frac{\rho L^3 \epsilon'}{12 \norm{\mathbf{\Delta}^\star}^3 (2\ell+\rho L)} = \frac{\rho \epsilon'}{\beta+\rho \norm{\mathbf{\Delta}^\star}} \cdot \frac{1}{12}  \underbrace{\frac{\beta+ \rho \norm{\mathbf{\Delta}^\star}}{2\ell+\rho L} \frac{L^3}{\norm{\mathbf{\Delta}^\star}^3}}_{\bar{\sigma}}. \notag
    \end{align}
    By concentration we have that $(1-c_2)\ell \leq \beta \leq (1+c_2) \ell$. So we can easily check that $\frac{1}{3} \frac{L^3}{\norm{\mathbf{\Delta}^\star}^3} \leq \bar{\sigma} \leq 1$ for sufficiently small $c_2$.
    We can further show that:
    \begin{align}
      \sigma = c_3 \frac{\rho^2 L^3}{24 \cdot 12 (2\ell+\rho L)} \leq c_3 \frac{\rho^2 L^3}{24 \cdot 12 \rho L} \leq c_3 \frac{\rho^2 L^3} {24 \cdot 12 \rho L} \leq c_3 \frac{\rho L^2}{288} \leq c_3 \frac{\epsilon}{2304}. \label{eq:sigma_mag}
    \end{align}
    Thus $R = \frac{\beta}{\rho} + \sqrt{(\frac{\beta}{\rho})^2 + \frac{2 \norm{\mathbf{g}}}{\rho}} \leq 2 \frac{\ell}{\rho}$ for sufficiently small $c_1$, $c_2$, and $c_4$.
    Accordingly, $\frac{1}{4 (\beta + \frac{\rho}{2} R)} \geq \frac{1}{20 \ell}$. We can similarly check this step size choice suffices for Lemma \ref{lem:final_grad_loop}.
    Let us then choose $\eta = \frac{1}{20 \ell}$ to satisfy \textbf{Assumption A} and $r = 0$ in accordance with \textbf{Assumption B} from \citet{carmon2016gradient}.

    Thus Theorem 3.2 from \citet{carmon2016gradient} shows that with probability at least $1-\delta'$ the iterates
    $\tilde{m}(\mathbf{\Delta}_t) \leq \tilde{m}(\mathbf{\Delta}^\star) + (1+\bar{\sigma}) \epsilon' \leq \tilde{m}(\mathbf{\Delta}^\star) + 2 \epsilon' = \tilde{m}(\mathbf{\Delta}^\star) + c_3 \frac{\rho}{12} \norm{\mathbf{\Delta}_t^\star}^3$
    for $t \geq \mathcal{T}(\epsilon)$.
    We can now upper bound $\mathcal{T}(\epsilon)$. Recall we set $L = \frac{1}{2} \sqrt{\frac{\epsilon}{\rho}}$ for clarity.
      \begin{align}
        & \mathcal{T}(\epsilon) = \frac{1+\bar{\sigma}}{\eta} \min \left \{ \frac{1}{\rho \norm{\mathbf{\Delta}^\star}-\gamma}, \frac{10 \norm{\mathbf{\Delta}^\star}^2}{\epsilon'} \right \} \times \notag \\
        & \left[ 6 \log \left( 1+\mathbb{I}_{\{ \gamma > 0 \}} \frac{3 \sqrt{d}}{\bar{\sigma} \delta} \right) + 14 \log \left( \frac{(\beta+\rho \norm{\mathbf{\Delta}^\star}) \norm{\mathbf{\Delta}^\star}^2}{\epsilon'} \right)  \right] \notag \\
        & \leq \mathcal{O}(1) \cdot \frac{\ell}{\rho \norm{\mathbf{\Delta}^\star}} \times \left[ 6 \log \left(1+\frac{\sqrt{d}}{\delta'}\right) + 6 \log \left(9 \frac{\norm{\mathbf{\Delta}^\star}^3}{L^3} \right) + 14 \log \left( \frac{24 \beta}{c_3 \rho \norm{\mathbf{\Delta}^\star}} + \frac{48}{c_4} \right) \right] \notag \\
        & \leq \mathcal{O}(1) \cdot \frac{\ell}{\rho \norm{\mathbf{\Delta}^\star}} \times \left[ \mathcal{O}(1) \cdot \log \left(1+\frac{\sqrt{d}}{\delta'}\right) + \mathcal{O}(1) \cdot \log \left(\frac{\norm{\mathbf{\Delta}^\star}}{L} \right) + \mathcal{O}(1) \cdot \log \left( \frac{\beta}{\rho \norm{\mathbf{\Delta}^\star}} + 1 \right) + \mathcal{O}(1) \right] \notag \\
        & \leq  \mathcal{O}(1) \cdot \frac{\ell}{\rho \norm{\mathbf{\Delta}^\star}} \times \left[ \mathcal{O}(1) \cdot \log \left(1+\frac{\sqrt{d}}{\delta'}\right) + \mathcal{O}(1) \cdot \log \left( \frac{\beta}{\rho L} + \frac{\norm{\mathbf{\Delta}^\star}}{L} \right) + \mathcal{O}(1) \right] \notag \\
        & \leq  \mathcal{O}(1) \cdot \frac{\ell}{\rho \frac{1}{2} \sqrt{\frac{\epsilon}{\rho}}} \times \left[ \mathcal{O}(1) \cdot \log \left(1+\frac{\sqrt{d}}{\delta'}\right) + \mathcal{O}(1) \cdot \log \left( \frac{\ell}{\rho \frac{1}{2} \sqrt{\frac{\epsilon}{\rho}}} + \frac{\norm{\mathbf{\Delta}^\star}}{\frac{1}{2} \sqrt{\frac{\epsilon}{\rho}}} \right) + \mathcal{O}(1) \right] \notag \\
        & \leq \mathcal{O}(1) \cdot \ell \sqrt{\frac{1}{\rho \epsilon}}  \times \left[ \mathcal{O}(1) \cdot \log \left(1+\frac{\sqrt{d}}{\delta'}\right) + \mathcal{O}(1) \cdot \log \left( \frac{\ell}{\rho \sqrt{\frac{1}{\epsilon \rho}}} + \frac{\ell}{\rho \sqrt{\frac{1}{\epsilon \rho}}} \right) + \mathcal{O}(1) \right] \notag \\
        & \leq \mathcal{O}(1) \cdot \ell \sqrt{\frac{1}{\rho \epsilon}}  \times \left[ \mathcal{O}(1) \cdot \log \left(1+\frac{\sqrt{d}}{\delta'}\right) + \mathcal{O}(1) \cdot \log \left( \ell \sqrt{ \frac{1}{\rho \epsilon}} \right) + \mathcal{O}(1) \right] \notag \\
        & \leq \mathcal{O}(1) \cdot \ell \sqrt{\frac{1}{\rho \epsilon}} \times \left[ \log \left( \ell \sqrt{\frac{1}{\rho \epsilon}} \left(1+\frac{\sqrt{d}}{\delta'} \right) \right) + \mathcal{O}(1) \right], \notag
      \end{align}
      where we have used the bound $\norm{\mathbf{\Delta}^\star} \leq \mathcal{O}(1) \cdot \frac{\ell}{\rho}$. We can see that $\norm{\mathbf{\Delta}^{\star}} \leq \mathcal{O}(1) \cdot \frac{\ell}{\rho}$ by appealing to the first-order stationary condition in Equation \eqref{eq:cubic_sol_stat}, $\norm{\mathbf{g}} \leq \frac{\ell^2}{\rho}$,
      and $\norm{\mathbf{B} \mathbf{\Delta}^\star} \leq (1+c_2) \ell \norm{\mathbf{\Delta}^\star}$.
      Combining these facts\footnote{The Fenchel-Young inequality $ab \leq a^2 + \frac{b^2}{4}$ is also used in the second line of this display.} we have:
      \begin{align}
        & \mathbf{g}+ \mathbf{B} \mathbf{\Delta^\star} + \frac{\rho}{2} \norm{\mathbf{\Delta}^\star} \mathbf{\Delta}^\star = 0 \implies \notag \\
        & \frac{\rho}{2} \norm{\mathbf{\Delta}^\star}^2 \leq \norm{\mathbf{g}} + \norm{\mathbf{B} \mathbf{\Delta}^\star} \leq \frac{\ell^2}{\rho} + (1+c_2) \ell \norm{\mathbf{\Delta}^\star} \leq \left( 1+(1+c_2)^2 \right) \frac{\ell^2}{\rho} + \frac{\rho}{4} \norm{\mathbf{\Delta}^\star}^2 \notag \implies \\
        & \norm{\mathbf{\Delta}^\star} \leq 3 \frac{\ell}{\rho}, \notag
      \end{align}
      for sufficiently small $c_2$. This completes the proof of the first statement of the Lemma.

      Note that we will eventually choose $\delta' \sim \mathcal{O}(\frac{1}{\epsilon^{1.5}})$ for our final guarantee. However, this will only contribute logarithmic dependence in
      $\epsilon$ to our upper bound.

      We now show that the iterate $\norm{\mathbf{\Delta}_{\mathcal{T}(\epsilon)}} \leq \norm{\mathbf{\Delta}^\star} + c_4 \sqrt{\frac{\epsilon}{\rho}}$. For notational convenience let us use $\norm{\tilde{\mathbf{\Delta}}^\star}$ to denote the norm of the global minima
      of the \textit{perturbed} subproblem. Since our step size choice and initialization satisfy \textbf{Assumptions A} and \textbf{B} then we can also apply Corollary 2.5 from \citet{carmon2016gradient}. Corollary 2.5 states the norms $\norm{\mathbf{\Delta}_t}$ are non-decreasing
      and satisfy $\norm{\mathbf{\Delta}_t} \leq \norm{\tilde{\mathbf{\Delta}}^\star}$. So we immediately obtain $\norm{\mathbf{\Delta}_{\mathcal{T}(\epsilon)}} \leq \norm{\tilde{\mathbf{\Delta}}^\star}$.
      We can then use Lemma 4.6 from \citet{carmon2016gradient} which relates the norm of the global minima of the \textit{perturbed} subproblem, $\norm{\tilde{\mathbf{\Delta}}^\star}$, to the norm of the global minima of the original problem, $\norm{\mathbf{\Delta}^\star}$.
      Lemma 4.6 from \citet{carmon2016gradient} states that under the gradient perturbation $\tilde{\mathbf{g}} = \mathbf{g} + \sigma q$ we have that $\abs{\norm{\tilde{\mathbf{\Delta}}^\star}^2 - \norm{\mathbf{\Delta}^\star}^2} \leq \frac{4 \sigma}{\rho}$.
      So using the upper bound on $\sigma$ from Equation \eqref{eq:sigma_mag} we obtain that:
      \begin{align}
        & \norm{\mathbf{\Delta}_{\mathcal{T}(\epsilon)}}^2 \leq \norm{\mathbf{\tilde{\Delta}}^\star}^2 \leq  \norm{\mathbf{\Delta}^\star}^2 + \frac{4\sigma}{\rho} \leq \norm{\mathbf{\Delta}^\star}^2 + \frac{c_3}{576} \frac{\epsilon}{\rho} \implies  \notag \\
        & \norm{\mathbf{\Delta}_{\mathcal{T}(\epsilon)}} \leq \norm{\mathbf{\Delta}^\star} + \sqrt{\frac{c_3}{576}} \sqrt{\frac{\epsilon}{\rho}} = \norm{\mathbf{\Delta}^\star} + c_4 \sqrt{\frac{\epsilon}{\rho}}, \notag
      \end{align}
      where we define $c_4 = \sqrt{\frac{c_3}{576}}$.
    \end{proof}

\subsection{Proofs of Lemma \ref{lem:subsolveerr} and Corollary \ref{thm:main2}}
Here we conclude by showing the correctness of Algorithm \ref{algo:SCD} which follows easily using our previous results.
\subsolveerr*
\begin{proof}
  This result follows immediately from Lemmas \ref{lem:cauchy} and \ref{lem:duchi_subsolver}. In particular, Lemma \ref{lem:cauchy} shows the Cauchy step, which is only used when $\norm{\mathbf{g}} \geq \frac{\ell^2}{\rho}$, satisfies the
  conditions of an \textbf{Case 1} procedure. Lemma \ref{lem:duchi_subsolver} shows solving the cubic submodel via gradient descent, which is only used when $\norm{\mathbf{g}} \leq \frac{\ell^2}{\rho}$, satisfies the conditions of a \textbf{Case 2} procedure.
  Lemma \ref{lem:cauchy} shows the Cauchy step has gradient complexity $\mathcal{O}(1)$. Lemma \ref{lem:duchi_subsolver} shows the gradient complexity of the gradient descent loop is upper bounded by $\tilde{\mathcal{O}} (\frac{\ell}{\sqrt{\rho \epsilon}})$ but has a failure probability
  $1-\delta'$ over the randomness in the gradient perturbation.
\end{proof}
Finally, assembling all of our results we can conclude that:
Using Lemma \ref{lem:subsolveerr} and Theorem \ref{thm:main1} we can immediately see that:
 \thmmaincarmon*
 \begin{proof}
   We can check the descent constants (with respect to the cubic submodel)
   referenced in the proof of Theorem \ref{thm:main1} for the \textbf{Case 1} and \textbf{Case 2} procedures of Algorithm \ref{algo:SCD}, $K_1=\frac{7}{20}$ and and $K_2 \leq \frac{1}{96}$ respectively, satisfy $-K_1 \leq -K_2$.
   The conclusion then follows immediately from Theorem \ref{thm:main1}, since Lemma \ref{lem:subsolveerr} shows that Algorithm \ref{algo:SCD} is a Cubic-Subsolver routine, as defined in Condition \ref{cond:oracle_append},
   with iteration complexity $\mathcal{T}(\epsilon) \leq \tilde{\mathcal{O}}(\frac{\ell}{\sqrt{\rho \epsilon}})$.
 \end{proof}

\section{Experimental Details}\label{apx:experimental-details}

\subsection{Synthetic Nonconvex Problem}

The W-shaped function used in our synthetic experiment is a piecewise cubic function defined in terms of a slope parameter $\epsilon$ and a length parameter $L$:
\begin{align*}
w(x) =
\begin{dcases}
\sqrt{\epsilon} \left(x + (L+1) \sqrt{\epsilon}\right)^2 - \frac{1}{3} \left(x + (L+1) \sqrt{\epsilon}\right)^3 - \frac{1}{3} (3 L + 1) \epsilon^{3/2} ,
& x \le -L \sqrt{\epsilon} ; \\
\epsilon x + \frac{\epsilon^{3/2}}{3} ,
& -L \sqrt{\epsilon} < x \le -\sqrt{\epsilon} ; \\
-\sqrt{\epsilon} x^2 - \frac{x^3}{3} ,
& -\sqrt{\epsilon} < x \le 0 ; \\
-\sqrt{\epsilon} x^2 + \frac{x^3}{3} ,
& 0 < x \le \sqrt{\epsilon} ; \\
-\epsilon x + \frac{\epsilon^{3/2}}{3} ,
& \sqrt{\epsilon} < x \le L \sqrt{\epsilon} ; \\
\sqrt{\epsilon} \left(x - (L+1) \sqrt{\epsilon}\right)^2 + \frac{1}{3} \left(x - (L+1) \sqrt{\epsilon}\right)^3 - \frac{1}{3} (3 L + 1) \epsilon^{3/2} ,
& L \sqrt{\epsilon} \le x . \\
\end{dcases}
\end{align*}
We set $\epsilon = 0.01$ and $L = 5$ in our experiment.

For stochastic cubic regularization, we fix $\rho = 1$ at the analytic Hessian Lipschitz constant for this problem, and we use 10 inner iterations for each invocation of the cubic subsolver, finding that this yields a good trade-off between progress and accuracy. Then for each method, we perform a grid search over the following hyperparameters:
\begin{itemize}
\item Batch size: $\{ 10, 30, 100, 300 \}$
\item Step size: $\{ c \cdot 10^{-i} : c \in \{1,3\}, i \in \{1,2,3,4,5\} \}$
\end{itemize}
Gradient and Hessian batch sizes are tuned separately for our method. We select the configuration for each method that converges to a global optimum the fastest, provided the objective value stays within 5\% of the optimal value after convergence. Since the global optima are located at $(\pm\frac{3}{5},0)$ and each has objective value $-\frac{2}{375}$, this is equivalent to an absolute tolerance of $\frac{1}{3750} = 0.0002666\cdots$.

\subsection{Deep Autoencoder}

Due to computational constraints, we were unable to perform a full grid search over all hyperparameters for every method. As a compromise, we fix all batch sizes and tune the remaining hyperparameters. In particular, we fix the gradient batch size for all methods at 100, a typical value in deep learning applications, and use a Hessian batch size of 10 for stochastic cubic regularization as motivated by the theoretical scaling. Then we perform a grid search over step sizes for each method using the same set of values from the synthetic experiment. We additionally select $\rho$ from $\{ 0.01, 0.1, 1 \}$ for stochastic cubic regularization, but find that the choice of this value had little effect on final performance. As in the synthetic experiments, we carry out 10 inner iterations per invocation of the cubic subsolver.

\end{document}